\pdfoutput=1
\documentclass[11pt]{article}
\usepackage{times}
\usepackage[letterpaper, left=1in, right=1in, top=1in, bottom=1in]{geometry}

\usepackage{microtype}
\usepackage{graphicx}
\usepackage{subfigure}
\usepackage{header}
\usepackage{booktabs} %

\usepackage{enumitem}

\usepackage{hyperref}

\usepackage{amsmath}
\usepackage{amssymb}
\usepackage{mathtools}
\usepackage{amsthm}
\usepackage{mathrsfs} 
\usepackage{algorithm}
\usepackage{algorithmic}
\usepackage{thmtools,thm-restate}
\allowdisplaybreaks

\usepackage[capitalize,noabbrev]{cleveref}

\theoremstyle{plain}
\newtheorem{theorem}{Theorem}%
\newtheorem{proposition}[theorem]{Proposition}
\newtheorem{lemma}[theorem]{Lemma}
\newtheorem{corollary}[theorem]{Corollary}
\theoremstyle{definition}
\newtheorem{definition}[theorem]{Definition}
\newtheorem{assumption}{Assumption}

\theoremstyle{remark}

\usepackage{natbib}
\usepackage[textsize=tiny]{todonotes}

\usepackage{hyperref,url}
\hypersetup{
  colorlinks,
  citecolor=blue!70!black,
  linkcolor=blue!70!black,
  urlcolor=blue!70!black}
\usepackage{cleveref}
\crefformat{equation}{#2(#1)#3}
\Crefformat{equation}{#2(#1)#3}

\def\shownotes{1}  %
\ifnum\shownotes=1
\newcommand{\authnote}[2]{{\scriptsize $\ll$\textsf{#1 notes: #2}$\gg$}}
\else
\newcommand{\authnote}[2]{}
\fi

\title{Offline Learning in Markov Games \\ with General Function Approximation}

\author{Yuheng Zhang\thanks{Department of Computer Science, University of Illinois Urbana-Champaign. Email:  \texttt{yuhengz2@illinois.edu, nanjiang@illinois.edu}.} \and Yu Bai\thanks{Salesforce Research. Email: \texttt{yu.bai@salesforce.com}} \and Nan Jiang\footnotemark[1]}

\renewcommand{\cite}{\citep}

\begin{document}
\maketitle

\begin{abstract}
We study offline multi-agent reinforcement learning (RL) in Markov games, where the goal is to learn an approximate equilibrium---such as Nash equilibrium and (Coarse) Correlated Equilibrium---from an offline dataset pre-collected from the game. 
Existing works consider relatively restricted tabular or linear models and handle each equilibria separately. In this work, we provide the first framework for sample-efficient offline learning in Markov games under general function approximation, handling all 3 equilibria in a unified manner. By using Bellman-consistent pessimism, we obtain interval estimation for policies' returns, and use both the upper and the lower bounds to obtain a relaxation on the gap of a candidate policy, which becomes our optimization objective. 
Our results generalize prior works and provide several additional insights. 
Importantly, we require a data coverage condition that improves over the recently proposed ``unilateral concentrability''. Our condition allows selective coverage of deviation policies that optimally trade-off between their greediness (as approximate best responses) and coverage, and we show scenarios where this leads to significantly better guarantees. 
As a new connection, we also show how our algorithmic framework can subsume seemingly different solution concepts designed for the special case of two-player zero-sum games. %
\end{abstract}

\section{Introduction}

Offline RL aims to learn a good policy from a pre-collected historical dataset. It has emerged as an important paradigm for bringing RL to real-life scenarios due to its non-interative nature, especially in applications where deploying adaptive algorithms in the real system is financially costly and/or ethically problematic~\citep{levine2020offline}.  %
While offline RL has been extensively studied in the single-agent setting, many real-world applications involve the strategic interactions between multiple agents. This renders the necessity of bringing in game-theoretic reasoning, often modeled using \textit{Markov games} \cite{shapley1953stochastic} in the RL theory literature. Markov games can be viewed as the multi-agent extension of Markov Decision Processes (MDPs), where agents share the same state information and the dynamics is determined by the joint action of all agents. 

While online RL in Markov games has seen significant developments in recent years \cite{bai2020provable,liu2021sharp,song2021can,jin2021v}, offline learning in Markov games has only started to attract attention from the community. Earlier works \cite{cui2022provably,zhong2022pessimistic} focus on tabular cases or linear function approximation, which cannot handle complex environments that require advanced function-approximation techniques. Although there has been a rich literature on single-agent RL with general function approximation \cite{jiang2017contextual,jin2021bellman,wang2020provably,huang2021going}, whether and how they can be extended to offline Markov games remains largely unclear. In addition, the learning goal in Markov games is no longer return optimization, but instead finding an \emph{equilibrium}. However, there are multiple popular notions of equilibria, and prior results for the offline setting mainly focus on one of them (Nash) \cite{cui2022offline,cui2022provably,zhong2022pessimistic}. 
These considerations motivate us to study the following question:

\begin{center}
\emph{Can we design sample-efficient algorithms for offline Markov games with general function approximation, and handle different equilibria in a unified framework?}
\end{center}

\paragraph{Unified framework}  In this paper, we provide information-theoretic results that answer the question in the positive. %
We first express the equilibrium gap---the objective we wish to minimize---in a unified manner for 3 popular notions of equilibria: Nash Equilibrium (NE), Correlated Equilibrium (CE), and Coarse Correlated Equilibrium (CCE) (Section~\ref{sec:eq}). Then, we build on top of the Bellman-consistent pessimism framework from single-agent offline RL \citep{xie2021bellman}, which allows us to construct confidence sets for policy evaluation and obtain the confidence intervals of policies' returns. %
An important difference is that \citet{xie2021bellman} only needs \emph{pessimistic} evaluations in the single-agent case; in contrast,  we need both \emph{optimistic} and \emph{pessimistic} evaluations to further compute a surrogate upper bound on the \emph{equilibrium gap} of each candidate policy, which provably leads to strong offline learning guarantees (Section~\ref{sec:multi_player}).  %

\paragraph{New insights on data conditions} Our algorithm and analyses also shed new light on the offline learnability of Markov games. In single-agent offline RL, it is understood that a good policy can be learned as long as the data covers one, and this condition is generally known as ``single-policy concentrability/coverage'' \cite{jin2021pessimism,zhan2022offline}. In contrast, in Markov games, data covering an equilibrium is intuitively insufficient, as a  fundamental aspect of equilibrium is reasoning about what would happen if other agents were to \emph{deviate}. To address this discrepancy,  a notion of ``unilateral concentrability'' is proposed as a sufficient data condition for offline Markov games \cite{cui2022offline} (see also~\citet{zhong2022pessimistic}), which asserts that the equilibrium as well as its all unilateral deviations are covered. %
While this is sufficient and in the worst-case necessary, it remains unclear whether less stringent conditions may also suffice.  
Our work %
relaxes the assumption and provide more flexible guarantees. Instead of depending on the worst-case estimation error of all unilateral deviation policies, our error bound exhibits the trade-off between a policy coverage error term and a policy suboptimality term. It automatically adapts to the optimal trade-off, %
and we show scenarios in \cref{app:discussion} where the bound significantly improves over unilateral coverage results \cite{cui2022provably}.

\paragraph{V-type variant} Our main algorithm estimates the policies' Q-functions, which takes all agents' actions as inputs. When specialized to the tabular setting, this would incur an exponentially dependence on the number of agents. While this can be avoided by using strong function approximation to generalize over the joint action space \cite{zhong2022pessimistic}, it prevents us from reproducing and subsuming the prior works \cite{cui2022offline,cui2022provably}. To address this issue, we propose a V-type variant of our algorithm, which estimates state-value functions instead and uses importance sampling to correct for action mismatches. It naturally avoids the exponential dependence, and reproduces the rate (up to minor differences) of \cite{cui2022provably} whose analysis is specialized to tabular settings (Section~\ref{sec:v_type}). 

\paragraph{New connection for two-player zero-sum games} As an additional discovery, we show interesting connection between our work and prior algorithmic ideas \citep{jin2022power, cui2022provably} that are specifically designed for two-player zero-sum games. While they seem very different at the first glance, we show in Appendix~\ref{app:discussion} that these ideas can be subsumed by our algorithmic framework and our analyses and guarantees extend straightforwardly. %

\subsection{Related Work}
\paragraph{Offline RL} Offline RL aims to learn a good policy from a pre-collected dataset without direct interaction with the environment. There are 
many prior works studying single-agent offline RL problem in both the tabular~\citep{yin2021nearp,yin2021nearo,yin2021towards,rashidinejad2021bridging,xie2021policy,shi2022pessimistic,li2022settling} and function approximation setting~\cite{antos2008learning,precup2000eligibility,chen2019information,xie2020q,xie2021batch,xie2021bellman,jin2021pessimism,zanette2021provable,uehara2021pessimistic,yin2022near,zhan2022offline}. Notably,~\citet{xie2021bellman} introduces the notion of Bellman-consistent pessimism and our techniques are built on it.
\paragraph{Markov games} Markov games is a widely used framework for multi-agent reinforcement learning. Online learning equilibria of Markov games has been extensively studied, including two-player zero-sum Markov games \cite{wei2017online,bai2020provable,bai2020near,liu2021sharp,dou2022gap}, and multi-player general-sum Markov games \cite{liu2021sharp,song2021can,jin2021v,mao2022provably}. Three equilibria are usually considered as the learning goal---Nash Equilibrium (NE), Correlated Equilibrium (CE) and Coarse Correlated Equilibrium (CCE). Recently, a line of works consider solving Markov games with function approximation, including linear %
\cite{xie2020learning,chen2022almost} and general function approximation \cite{huang2021towards,jin2022power}. A closely related work is~\citet{jin2022power}, where a multi-agent version of the Bellman-Eluder dimension is introduced to solve zero-sum Markov games under general function approximation. However, they focus on the online setting which is different from our offline setting.
\paragraph{Offline Markov games} Since \citet{cui2022offline}'s initial work on offline tabular zero-sum Markov games, there have been several follow-up works on offline Markov games, either for tabular zero-sum / general-sum Markov games \citep{cui2022provably,yan2022model} or linear function approximation \citep{zhong2022pessimistic,xiong2022nearly}. In this work, we study general function approximation for multi-player general-sum Markov games, which is a more general framework. Technically, we differ from these prior works in how we handle uncertainty quantification in policy evaluation, an important technical aspect of offline learning: we use initial state optimism/pessimism for policy evaluation, whereas previous works rely on pre-state pessimism with bonus terms. In addition, previous works require the so-called ``unilateral concentrability'' assumption of data coverage.\footnote{\citet{zhong2022pessimistic} proposes the notion of ``relative uncertainty, which is the linear version of ``unilateral concentrability''.} Although this assumption is unavoidable for the worst-case, our approach requires a condition that is never worse (and coincides in the worst-case) and can be significantly better on certain instances.

\section{Preliminaries}

\paragraph{Notations} We use $\Delta(\cdot)$ to denote the probability simplex. We use bold letters to denote vectors such as $\bfa$ and the $j^{\mathrm{th}}$ element of $\bfa$ is denoted by $\bfa_j$. We use $-i$ to denote all the players except player $i$. For a positive integer $m$, $[m]$ denotes the set $\{1,2,\cdots,m\}$. $\|f\|_{2,d}^2$ represents $\E_d[f^2]$ and $f(s,\pi)$ stands for $\E_{a \sim \pi(\cdot|s)}[f(s,a)]$. We use $\order(\cdot)$ to hide absolute constants and use $\widetilde{\order}(\cdot)$ to hide logarithmic factors.

\subsection{Multi-player General-sum Markov Games}
We consider multi-player general-sum Markov games in the infinite-horizon discounted setting. Such a Markov game is specified by $(\Scal,\Acal=\prod_{i \in [m]}\Acal_i,P,r,\gamma,s_0)$, where $\Scal$ is the state space with $|\Scal|=S$, $\Acal_i$ is the action space for player $i$ with $|\Acal_i|=A_i$, $\bfa \in \Acal$ is the joint action taken by all $m$ players, $P:\Scal \times \Acal \rightarrow \Delta(\Scal)$ is the transition function and $P(\cdot|s,\bfa)$ describes the probability distribution over the next state when joint action $\bfa$ is taken at state $s$, $r=\{r_i\}_{i \in [m]}$ is the collection of reward function where $r_i: \Scal \times \Acal \rightarrow [0,\rmax]$ is the deterministic reward function for player $i$, $\gamma \in [0,1)$ is the discount factor, and $s_0$ is the fixed initial state which is without loss of generality. 

\paragraph{Product and correlated policies} A Markov joint policy $\pi:\Scal \rightarrow \Delta(\Acal)$ specifies the decision-making strategies of all players and induces a trajectory $s_0,\bfa_0,\bfr_0,s_1,\bfa_1,\bfr_1, \cdots, s_t,\bfa_t, \bfr_t, \cdots$, where $\bfa_t \sim \pi(\cdot|s_t)$, $\bfr_{t,i}=r_i(s_t,\bfa_t)$, and $s_{t+1} \sim P(\cdot|s_t,\bfa_t)$. 
For a joint policy  $\pi$, $\pi_i$ is the marginalized policy of player $i$ and $\pi_{-i}$ is the marginalized policy for the remaining players. A joint policy $\pi$ is a product policy if $\pi=\pi_1 \times \pi_2 \times \cdots \pi_m$ where each player $i$ takes actions independently according to $\pi_i$. If $\pi$ is not a joint policy, sometimes  we say $\pi$ is  correlated, and the players need to depend their actions on public randomness. 

\paragraph{Value function and occupancy} For player $i$ and joint policy $\pi$, we define the value function $V^\pi_i(s) \coloneqq \E_\pi[\sum_{t=0}^{\infty} \gamma^t r_i(s_t,\bfa_t)|s_0=s]$ and the Q-function $Q^\pi_i(s,\bfa) \coloneqq \E_\pi[\sum_{t=0}^{\infty} \gamma^t r_i(s_t,\bfa_t)|s_0=s,a_0=\bfa]$, they are bounded in $[0,\Vmax]$ where $\Vmax=\rmax/(1-\gamma)$. 
For each joint policy $\pi$, the policy-specific Bellman operator of the $i^{\mathrm{th}}$ player $\Tcal^\pi_i:\R^{\Scal \times \Acal} \rightarrow \R^{\Scal \times \Acal}$ is defined as
\begin{align*}
(\Tcal^\pi_i f)(s,\bfa)=r_i(s,\bfa)+\gamma \E_{s' \sim P(\cdot|s,\bfa)}[f(s',\pi)],
\end{align*}
and $Q^\pi_i$ is the unique fixed point of $\Tcal^\pi_i$. Note that once a policy is fixed, the game-theoretic considerations are no longer relevant and the value functions are defined in familiar manners similar to the single-agent setting, with the only difference that each player $i$ has its own value function due to the player-specific reward function $r_i$. 
Similar to the single-agent case, we also consider the discounted state-action occupancy $d^\pi(s,\bfa) \in \Delta(\Scal \times \Acal)$ which is defined as $d^\pi(s,\bfa)=(1-\gamma)\E_\pi[\sum_{t=0}^{\infty} \gamma^t \mathbb{I}[s_t=s,\bfa_t=\bfa]]$.

\subsection{Offline learning of Markov games}
In the offline learning setting, we assume access to a pre-collected dataset and cannot have further interactions with the environment. The offline dataset $\Dcal$ consists of $n$ independent tuples $(s,\bfa,\bfr,s')$, which are generated as $(s,\bfa) \sim d_D$, $\bfr_i \sim r_i(s,\bfa)$ and $s' \sim P(\cdot|s,\bfa)$ with some data distribution $d_D \in \Delta(\Scal\times\Acal)$.\footnote{For non-i.i.d.~adaptive data we may use martingale concentration inequalities in our analyses. Without further mixing-type assumptions, our analyses extend if we change the $d_D$ (which is a static object) in the definitions such as \cref{eq:complete} and \cref{eq:coverage} to $\widehat d_D$, which is the empirical distribution over state-action pairs. The resulting definition of \cref{eq:coverage}, for example, corresponds to quantities like $\widehat C(\pi)$ in \citet[Definition 3]{cui2022provably} defined for the tabular setting.}

\paragraph{Policy class} In practical problems with large state spaces, the space of all possible Markov joint policies is prohibitively large and intractable to work with. To address this, we assume we have a pre-specified policy class $\Pi \subset (\Scal\to\Delta(\Acal))$, from which we seek a policy that is approximately an equilibrium under a given criterion.\footnote{We only consider minimizing equilibrium gaps among a class of stationary Markov policies in this paper. See \citet{daskalakis2022complexity} and the references therein for how they suffice for standard notions of equilibria such as NE and CCE, and \citet{nowak1992existence} for the case of CE. Below we also only consider stationary Markov policies as \emph{response policies} for NE/CCE, which is also justified by the fact that once a stationary Markov $\pi_{-i}$ is fixed, optimizing player $i$'s behavior for best response becomes a single-agent MDP problem.}   Let $\Pi_i = \{\pi_i: \pi\in\Pi\}$ denote the class of induced marginalized policies for player $i$, and define $\Pi_{-i}$ similarly.

\paragraph{The extended class} As we will see in Section~\ref{sec:eq}, a fundamental aspect of  equilibria is the counterfactual reasoning of how other agents would deviate and respond to a given policy. After considering the possible deviation behaviors of player $i$ in response to each policy $\pi \in \Pi$, we arrive at an extended class $\exti \supseteq \Pi$ for player $i$. The concrete form of $\exti$ will be defined in \cref{sec:eq} and can depend on the notion of equilibrium under consideration, and for now it suffices to say that $\exti$ is a superset of $\Pi$ consisting of all policies that player $i$ needs to reason about.

\textbf{Value-function approximation} We use $\Fcal_i \subset (\Scal \times \Acal \rightarrow [0,\Vmax])$ to approximate the Q-function $Q^\pi_i$ for each player $i$. Following \cite{xie2021bellman}, we make two standard assumptions on $\Fcal_i$,
\begin{assumption}[Approximate  Realizability]\label{asm:realizability}
For any player $i \in [m]$ and any $\pi \in \exti$, we have
\begin{align*}
\inf_{f \in \Fcal_i}\sup_{\text{admissible } d} \left\|f - \Tcal_i^{\pi} f\right\|_{2,d}^2 \leq \varepsilon_\Fcal,
\end{align*}  
A data distribution $d$ is admissible if $d \in\{d^{\pi'}:\pi' \in \exti\} \cup d_D$.
\end{assumption}

For each player $i$ and joint policy $\pi$, \cref{asm:realizability} requires that there exists $f \in \Fcal_i$ such that $f$ has small Bellman error under all possible distributions induced from the extended policy class $\exti$ and the data distribution. When $Q^\pi_i \in \Fcal_i$, $\forall \pi \in \Pi, i \in [m]$, we have $\epsf=0$.
\begin{assumption}[Approximate  Completeness]\label{asm:completeness}
For any player $i \in [m]$ and any $\pi \in \exti$, we have
\begin{align} \label{eq:complete}
\sup_{f \in \Fcal_i} \inf_{f' \in \Fcal_i} \left\|f' - \Tcal_i^{\pi} f \right\|_{2,d_D}^2 \leq \varepsilon_{\Fcal,\Fcal}.
\end{align}
\end{assumption}
\cref{asm:completeness} requires that $\Fcal_i$ is approximately closed under operator $\Tcal_i^\pi$. Both assumptions are direct extensions of their counterparts that are widely used in the offline RL literature.

\paragraph{Distribution mismatch and data coverage} Similar to \citet{xie2021bellman}, we use the discrepancy of Bellman error under $\pi$ to measure the distribution mismatch between an arbitrary distribution $d$ and data distribution $d_D$:
\begin{align}\label{eq:coverage}
\Cscr(d;d_D,\Fcal_i, \pi) \coloneqq \max_{f \in \Fcal_i}\frac{\|f - \Tcal_i^\pi f\|_{2,d}^2}{\|f - \Tcal_i^\pi f\|_{2,d_D}^2}.
\end{align}
We remark that $\Cscr(d;d_D,\Fcal_i, \pi) \le \sup_{s,\bfa} \frac{d(s,\bfa)}{d_D(s,\bfa)}$, which implies that $\Cscr(d;d_D,\Fcal_i, \pi)$ is a tighter measurement than the raw density ratio.   %
\section{Equilibria}\label{sec:eq}
We consider three common equilibria in game theory: Nash Equilibrium (NE), Correlated Equilibrium (CE) and Coarse Correlated Equilibrium (CCE). We define the three equilibria in a unified fashion using the concept of \emph{response class mappings},  so that each equilibrium is defined with respect to the relative best response within each corresponding response class. 

A response class mapping $\res(\cdot)$ maps a policy $\pi$ to a policy \emph{class}, $\res(\pi)\coloneqq \bigcup_{i\in[m]} \resi(\pi)$. Roughly speaking, $\resi(\pi)$ is obtained by taking a candidate policy $\pi$,  considering various ways that player $i$ would deviate its behavior from $\pi_i$ to $\pi_i^\dagger$, and re-combining $\pi_i^\dagger$ and $\pi_{-i}$ into joint policies.\footnote{For this reason, the policy class $\resi(\pi)$ always satisfies the following:  for any $i\in[m]$ and any $\pi'\in\resi(\pi)$, $\pi'_{-i}=\pi_{-i}$.}  The space of possible $\pi_i^\dagger$ which player $i$ can choose from determines the mapping, and will take different forms under different notions of equilibria, as explained next.

\begin{enumerate}[leftmargin=*]
\item A product policy is NE if it satisfies that no player can increase her gain by deviating from her own policy. Therefore, the response class mapping for NE is defined as $\resne(\pi):=\{\resnei(\pi)\}_{i \in [m]}$, where $\resnei(\pi):=\{\pi_i^\dagger\times \pi_{-i}: \pi_i^\dagger \in \Pi_i\}$. Note that here $\resnei$ has no dependence on the input $\pi_i$, and player $i$ simply considers using some  $\pi_i^\dagger \in \Pi_i$ to replace $\pi_i$. 
\item A CE is defined by a class of \emph{strategy modifications} $\Phi=(\Phi_i)_{i\in[m]}$, where $\Phi_i\subseteq (\Scal \times \Acal_i \rightarrow \Acal_i)$ is a set of strategy modifications of the $i^{\mathrm{th}}$ player, and each $\phi_i \in \Phi_i$ is a mapping $\phi_i : \Scal \times \Acal_i \rightarrow \Acal_i$. For any joint policy $\pi$, the modified policy $\phi_i \diamond \pi$ is defined as: at state $s \in \Scal$, all players sample $\mathbf{a} \sim \pi(\cdot|s)$, the $i^{\mathrm{th}}$ player changes action $\bfa_i$ to $\phi_i(s,a_i)$ and $\bfa_{-i}$ remains the same. For CE, the response class mapping of each joint policy $\pi$ is defined as $\res(\pi) \coloneqq \{\resi(\pi)\}_{i \in [m]}$, where $\resi(\pi)=\{(\phi_i \diamond \pi_i) \odot \pi_{-i}: \phi_i \in \Phi_i\}$. 
\item CCE is defined for general (i.e., possibly correlated) joint policies and is a relaxation of NE. The only difference is that CCE does not require the candidate policy $\pi$ to be a product policy. Hence, the response class mapping of CCE is the same as that of NE. 
\end{enumerate}

With the definition of response class mapping, for $\eq \in \{\mathrm{NE}, \ce, \cce\}$, we define the gap of any joint policy $\pi$ with respect to $\reseq(\cdot)$ as
\begin{align*}
\mathrm{Gap}^{\reseq}(\pi) \coloneqq \max_{i \in [m]} \max_{\pi^\dagger\in\reseqi(\pi)} V_i^{\pi^\dagger}(s_0)-V_i^\pi(s_0).
\end{align*}
Now we are ready to present the definitions of three equilibria.
\begin{definition}[Equilibria; NE, CE, and CCE]\label{def:eq}
For $\eq \in \{\mathrm{NE}, \ce, \cce\}$, a joint policy (product for NE) is an $\eps$-EQ with respect to $\reseq(\cdot)$, if for the response class $\reseqi(\pi)$,
\begin{align*}
    \mathrm{Gap}^{\reseq}(\pi) \le \eps.
\end{align*}
\end{definition}
\cref{def:eq} is defined with respect to the policy class $\Pi$ and strategy modification class $\Phi$ (for CE). Throughout the paper, we focus on the theoretical guarantees of such ``in-class'' notion of  gaps, which is a reasonable definition if we assume that all players have limited representation power and must work with restricted policy classes. Under additional assumptions (which we call ``strategy completeness''; see Appendix~\ref{app:strategy_complete}), such ``in-class'' gaps can be related to a stronger notion of gap where unrestricted deviation policies are considered for the best response. %

With the response class mappings, we also define the extended policy class $\exti \coloneqq (\bigcup_{\pi\in\Pi} \resi(\pi)) \bigcup \Pi$, which characterizes all possible policies with deviation from the $i^{\mathrm{th}}$ player. In addition, we define $\ext \coloneqq \bigcup_{i=1}^m \exti$. %

\section{Information-Theoretic Results for Multi-player General-sum Markov Games}\label{sec:multi_player}
\subsection{Algorithm} \label{sec:BCEL}
As our learning goal is to find a policy $\pi\in\Pi$ with small equilibrium gap $\mathrm{Gap}^{\reseq}(\pi)$ (for EQ$\in\{\textrm{NE, CE, CCE}\}$), a natural idea is to simply estimate the gap and minimize it over $\pi\in\Pi$. Unfortunately, we are in the offline setting and only have access to data sampled from an arbitrary data distribution $d_D$, which may not provide enough information for evaluating the gap of certain policies. 

Since the gap is not always amendable to estimation, we instead seek a surrogate objective that will always be an \emph{upper bound} on 
the equilibrium gap of each candidate policy $\pi \in \Pi$. The upper bound should also be \emph{tight} when the policy is covered by the data and we have sufficient information to determine its gap accurately. To achieve this goal, we recall the definition of gap:
$$
\mathrm{Gap}^{\reseq}(\pi) \coloneqq \max_{i \in [m]} \max_{\pi^\dagger\in \reseqi(\pi)} V_i^{\pi^\dagger}(s_0)-V_i^\pi(s_0).
$$
The key idea in our algorithm is that 
$$
V_i^{\pi^\dagger}(s_0)-V_i^\pi(s_0) \le \ov_i^{\pi^\dagger}(s_0)-\uv_i^\pi(s_0),
$$
where 
\begin{itemize}[leftmargin=*]
\item $\ov_i^{\pi^\dagger}(s_0) \ge V_i^{\pi^\dagger}(s_0)$ is an \emph{optimistic} evaluation of $\pi^\dagger$.
\item $\uv_i^\pi(s_0) \le V_i^\pi(s_0)$ is an \emph{pessimistic} evaluation of $\pi$.
\end{itemize}
With this relaxation, the problem reduces to optimistic and pessimistic policy evaluation, for which we can borrow existing techniques from single-agent RL. 

\begin{algorithm}[t]
\caption{
Bellman-Consistent Equilibrium Learning (\BCEL) from an Offline Dataset
\label{alg:multi_player}
}
\begin{algorithmic}[1]
\STATE {\bfseries Input:} Offline dataset $\Dcal$, parameter $\epsv$, equilibrium $\mathrm{EQ} \in \{\mathrm{NE},\mathrm{CE},\mathrm{CCE}\}$
\STATE For each player $i \in [m]$ and policy $\pi \in \exti$, construct function version space %
\begin{align}\label{eq:version_space}
\Fcal_i^{\pi,\epsv}=\{f_i \in \Fcal_i:\Ecal_i(f_i,\pi;\Dcal) \le \epsv\}.
\end{align}
\STATE For each player $i \in [m]$, compute 
\begin{align}
\ov_i^{\pi^\dagger}(s_0)&=\max_{f\in \Fcal_i^{\pi^\dagger,\epsv}} f(s_0,\pi^\dagger), \quad \forall \pi^\dagger \in \exti. \label{eq:estimated_ov}\\
\uv_i^\pi(s_0)&=\min_{f\in \Fcal_i^{\pi,\epsv}} f(s_0,\pi), \quad \forall \pi \in \Pi.\label{eq:estimated_uv}
\end{align}
\STATE For each policy $\pi \in \Pi$, compute the estimated gap 
\begin{align}\label{eq:surrogate_eq}
\widehat{\mathrm{Gap}}_{\mathrm{EQ}}(\pi) \coloneqq \max_{i \in [m]}\max_{\pi^\dagger \in \reseqi(\pi)}\ov^{\pi^\dagger}_i(s_0)-\uv_i^\pi(s_0).
\end{align}
\vspace{-1.5em}
\STATE Output $\hatpi \leftarrow \min_{\pi \in \Pi} \widehat{\mathrm{Gap}}_{\mathrm{EQ}}(\pi)$.
\end{algorithmic}
\end{algorithm}

\paragraph{Bellman-consistent pessimism \& optimism} We use the Bellman-consistent pessimism framework from \citet{xie2021bellman} to construct optimistic and pessimistic policy evaluations. 
For each player $i$, we first use dataset $\Dcal$ to compute an empirical Bellman error of all function $f_i \in \Fcal_i$ under Bellman operator $\tpii$,
\begin{align*}
& \Ecal_i(f_i,\pi;\Dcal) \coloneqq \Lcal_i(f_i,f_i,\pi;\Dcal) - \min_{f'_i \in \Fcal_i}\Lcal_i(f'_i,f_i,\pi;\Dcal), \\
& \Lcal_i(f'_i,f_i, \pi;\Dcal) \coloneqq \frac{1}{n} \! \sum_{(s,\bfa,\bfr,s') \in \Dcal} \!\! \left(f'_i(s,\bfa) - \bfr_i - \gamma f_i(s',\pi) \right)^2.
\end{align*}
Similar to the single-agent setting, %
$\Ecal_i(f_i,\pi;\Dcal)$ is a good approximation of the true Bellman error of $f_i$ w.r.t.~$\pi$, i.e., $\Ecal_i(f_i,\pi;\Dcal) \approx \|f_i - \Tcal_i^\pi f_i\|_{2, d_D}^2$, so we can construct a version space $\Fcal_i^{\pi,\epsv}$ for each player $i$ and policy $\pi \in \exti$ in \cref{eq:version_space}. To ensure that the best approximation of $Q_i^\pi$ is contained in $\Fcal_i^{\pi,\epsv}$, given a failure probability $\delta>0$, we pick the threshold parameter $\epsv$ as follows,
\begin{align*}
\epsv = \frac{80\Vmax^2\log\frac{|\Fcal||\ext|}{\delta}}{n}+30\epsf,
\end{align*}
where $\Fcal = \bigcup_{i=1}^m \Fcal_i$. Then, optimistic and pessimistic evaluations can be obtained by  simply taking the highest and the lowest prediction on the initial state $s_0$ across all functions in the version space (\cref{eq:estimated_ov} and \eqref{eq:estimated_uv}).

With $\ov_i^{\pi^\dagger}(s_0)$ and $\uv_i^\pi(s_0)$ at hand, we calculate the estimated gap $\widehat{\mathrm{Gap}}_{\mathrm{EQ}}(\pi)$ for each $\pi \in \Pi$ in \cref{eq:surrogate_eq}. We select the policy $\hatpi$ with the lowest estimated gap and the algorithm is summarized in \cref{alg:multi_player}. 

\subsection{Theoretical guarantees}
Before presenting the theoretical guarantee, we introduce the interval width $\unc_i^\pi$ of $\Fcal_i^{\pi,\epsv}$, which will play a key role in our main theorem statement: 
\begin{align*}
\unc_i^\pi &\coloneqq \max_{f_i \in \Fcal_i^{\pi,\epsv}} f_i(s_0,\pi)-\min_{f_i \in \Fcal_i^{\pi,\epsv}} f_i(s_0,\pi).
\end{align*}
As we will see, $\unc_i^\pi$ is a measure of how well the data distribution $d_D$ covers $d^\pi$, the state-action occupancy of $\pi$. The better coverage, the smaller $\Delta_i^\pi$. This is formalized by the following proposition:
\begin{restatable}[Bound on interval width]{proposition}{upperdelta}
\label{prop:delta}
With probability at least $1-\delta$, for any player $i \in[m]$ and any $\pi \in \exti$, we have
\begin{align}\label{eq:Delta_i_pi}
& \unc_i^\pi \le \min_d\frac{1}{1-\gamma}\sqrt{\Cscr(d;d_D,\Fcal_i,\pi)}\eps_{\mathrm{apx}}+\frac{1}{1-\gamma} \sum_{s,\bfa}(d^{\pi}\!\setminus\! d)(s,\bfa) \left[\Delta f_i^{\pi}(s,\bfa) - \gamma(P^\pi\Delta f_i^{\pi})(s,\bfa)\right], 
\end{align}
where $\eps_{\mathrm{apx}}=\order\left(\Vmax\sqrt{\frac{\log \frac{|\Fcal||\ext|}{\delta}}{n}}+\sqrt{\epsf+\epsff}\right)$, $(d^\pi \setminus d)(s,\bfa) \coloneqq \max(d^\pi(s,\bfa)-d(s,\bfa),0)$, $\Delta f^\pi_i(s,\bfa) \coloneqq f_i^{\pi,\max}(s,\bfa)-f_i^{\pi,\min}(s,\bfa)$, and $(P^\pi f)(s,\bfa)=\E_{s' \sim P(\cdot|s,\bfa)}[f(s',\pi)]$.
\end{restatable}

Here, a distribution $d \in \Delta(\Scal \times \Acal)$ in \cref{eq:Delta_i_pi} is introduced to handle the discrepancy between $d_D$ and $d^\pi$. The first term in \cref{eq:Delta_i_pi} captures the distribution mismatch between $d$ and $d_D$, and the second term represents the off-support Bellman error under $\pi$. When the data distribution $d_D$ has a full coverage on $d^\pi$, $d$ can be chosen as $d^\pi$ and the second term becomes zero. Therefore, for the purpose of developing intuitions, one can always choose $d = d^\pi$ and treat $\Delta_i^\pi \propto \sqrt{\Cscr(d^\pi;d_D,\Fcal_i,\pi)}$, though in general some $d\ne d^\pi$ may achieve a better trade-off and tighter bound.

With an intuitive understanding of $\Delta_i^\pi$, we are ready to show the following theorem for our proposed algorithm.
\begin{restatable}{theorem}{multi}\label{thm:multi_bound}
With probability at least $1-\delta$, for any $\pi \in \Pi$ and $\mathrm{EQ} \in \{\mathrm{NE},\mathrm{CE},\mathrm{CCE}\}$, the output policy $\hatpi$ of \cref{alg:multi_player} satisfies that
\begin{align*}
    & \mathrm{Gap}^{\reseq}(\hatpi) \le \mathrm{Gap}^{\reseq}(\pi)+\frac{4\sqrt{\epsf}}{1-\gamma} + \max_{i \in [m]}\min_{\tpi_i \in \reseqi(\pi)} \left(\unc_{i}^{\tpi_i}+\unc_i^\pi+
     \mathrm{subopt}_i^\pi(\tpi_i)
     \right),
\end{align*}
where $\mathrm{subopt}_i^\pi(\tpi_i)\coloneqq\max_{\pid \in \resi(\pi)} \ov_i^{\pid}(s_0)-\ov_i^{\tpi_i}(s_0)$. 
\end{restatable}

\begin{figure*}[t]
\centering
\includegraphics[width=.8\textwidth]{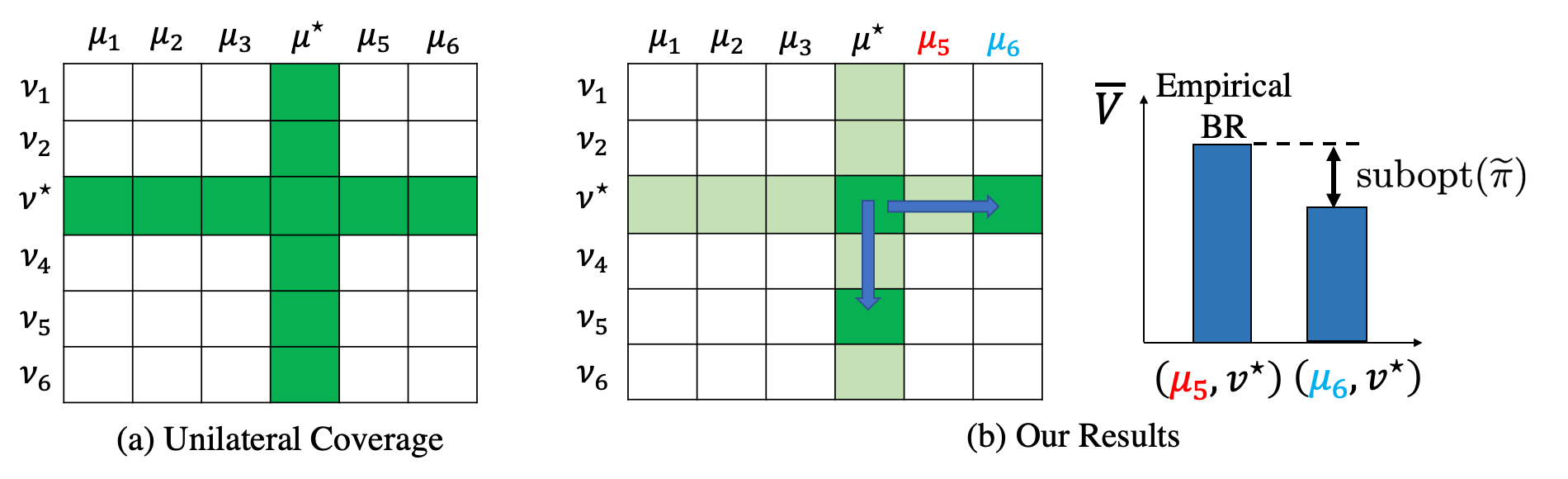}
\caption{
Illustration of unilateral coverage and our results on a zero-sum example. {\bf (a)} Unilateral coverage requires the dataset to cover all unilateral pairs $(\mu^\star, \nu')$ and $(\mu',\nu^\star)$ where $(\mu^\star,\nu^\star)$ is NE. {\bf (b)} Our approach enjoys an adaptive property and relaxes the condition. %
To begin with, we can already achieve a good sample complexity if the data were to cover the optimistic best response ($(\mu_5,\nu^\star)$ in this example) only, i.e. when $\unc^{\mu_5,\nu^\star}$ was small. Even when the dataset has a poor coverage on $(\mu_5,\nu^\star)$, there may exists some other $\mu_6$ so that $\unc^{\mu_5,\nu^\star} \gg \unc^{\mu_6,\nu^\star}$. Instead of suffering $\unc^{\mu_5,\nu^\star}$, our approach automatically adapts to the policy $\tpi=(\mu_6,\nu^\star)$ which achieves a better trade-off between the policy coverage term $\unc^{\tpi}$ and suboptimality term $\mathrm{subopt}(\tpi)$.
}
\label{fig:coverage}
\end{figure*}

\subsection{Improvement over unilateral coverage}\label{sec:improve}
To interpret Theorem~\ref{thm:multi_bound} and compare it to existing guarantees, we first introduce a direct corollary of  Theorem~\ref{thm:multi_bound} + \cref{prop:delta}, which is a relaxed form of our result that is closer to existing guarantees by \citet{cui2022offline, cui2022provably}. 

\begin{corollary}
\label{cor:unilateral}
For Nash equilibrium policy $\pis \in \Pi$, suppose there exists an unilateral coefficient $C(\pis)$ such that the following inequality holds
\begin{align}\label{eq:unilateral}
\max_{i \in [m]} \max_{\pi^\dagger \in \resnei(\pis)} \Cscr(d^{\pi^\dagger};d_D,\Fcal_i,\pis) \le C(\pis).
\end{align}
With probability at least $1-\delta$, we have
\begin{align*}
\mathrm{Gap}^{\resne}(\hatpi) \le \order\left(\frac{\Vmax\sqrt{\frac{\log \frac{|\Fcal||\ext|}{\delta}}{n}}+\sqrt{\epsf+\epsff}}{1-\gamma}\sqrt{C(\pis)}\right).
\end{align*}
\end{corollary}
The gap bound in Corollary~\ref{cor:unilateral} takes a simple form: the first part of it has an $O(1/\sqrt{n})$ statistical error (scaled by the complexities of function and policy classes, as fully expected), and an approximation error term that depends on $\epsf, \epsff$, which goes to $0$ when our function classes are exactly realizable and Bellman-complete. 

The key item in the bound is the $\sqrt{C(\pi^*)}$ factor, which measures distribution mismatch and implicitly determines the data coverage condition. $C(\pi^*)$ is defined in \cref{eq:unilateral}. As we can see, having a small $C(\pi^*)$ requires that data not only covers $\pi^*$ itself\footnote{Note that $\pi^* \in \resnei(\pis)$.}, but also \emph{all policies in} $\resnei(\pis) = \{\pi_i \times \pis_{-i}:\pi_i \in \Pi_i\}$. This is the notion of \emph{unilateral coverage} proposed by \citet{cui2022offline} and \citet{zhong2022pessimistic}. Visualizing this in Figure~\ref{fig:coverage}(a) with a simplified setting of a two-player matrix game, such a condition corresponds to data covering the entire ``cross'' centered at the NE.  

Although \cite{cui2022offline} argues that unilateral coverage is ``sufficient and necessary'' in the worst case, %
their argument does not exclude an  improved version that can be substantially relaxed under certain conditions, and we show that our Theorem~\ref{thm:multi_bound} is such a version. We now provide a breakdown of the bound in Theorem~\ref{thm:multi_bound}:
\begin{enumerate}[leftmargin=*, topsep=0pt, itemsep=0pt]
\item First, the RHS of the bound depends on $\Delta_i^\pi$, where $\pi$ is the policy we compete with and correspond to $\pi^*$ in Corollary~\ref{cor:unilateral}. Recalling that $\Delta_i^\pi \propto \sqrt{\Cscr(d^\pi;d_D,\Fcal_i,\pi)}$, this term corresponds to data coverage on $\pi^*$, which is always needed if we wish to compete with $\pi^*$.
\item The RHS also depends on $\Delta_i^{\tpi_i} + \mathrm{subopt}_i^\pi(\tpi_i)$, where $\tpi_i$ is \textit{minimized} over $\resnei(\pis)$ when EQ=NE (and \cref{eq:unilateral} \textit{maximizes} over $\pi^\dagger$). In particular, we can always choose $\tpi_i$ as the policy that maximizes $\ov_i$, i.e., the \emph{optimistic best response}. This would set $\mathrm{subopt}_i^\pi(\tpi_i)=0$, showing that we only need coverage for the optimistic best response policy, instead of all policies in  $\resnei(\pis)$  as required by the unilateral assumption.
\item Finally, our bound provides a further relaxation: when the optimistic best response is poorly covered, we may choose some other well-covered $\tpi_i$ instead, and pay an extra term $\mathrm{subopt}_i^\pi(\tpi_i)$ that measures to what extent $\tpi_i$ is an approximate $\ov_i$-based best response. 
\end{enumerate}
Again, we illustrate the flexibility of our bound in Figure~\ref{fig:coverage}(b). Below we also show a concrete example, where our guarantee leads to significantly improved sample rates compared to that provided by the unilateral condition. %

\paragraph{Example} 
Consider a simple two-player zero-sum matrix game with payoff matrix:
\begin{align*}
\centering
\begin{tabular}{llll}
    & $b_1$ & $b_2$ & $b_3$ \\
$a_1$ & 0.5 & 0.75 & 0.75 \\
$a_2$ & 0.25 & 0 & 0 \\
$a_3$ & 0.25 & 0 & 0 \\
\end{tabular}
\end{align*}
where the column player aims to maximize the reward and the row player aims to minimize it. It is clear to see $(a_1,b_1)$ is NE. The offline dataset $\Dcal$ is collected from the following distribution,
\begin{align*}
\centering
\begin{tabular}{llll}
    & $b_1$ & $b_2$ & $b_3$ \\
$a_1$ & $p_1$ & $p_2$ & $p_2$ \\
$a_2$ & $p_2$ & $p_3$ & $p_3$ \\
$a_3$ & $p_2$ & $p_3$ & $p_3$ \\
\end{tabular}
\end{align*}
where $0<p_2 \ll p_1$ and $p_3=\frac{1-p_1-4p_2}{4}$. Under Corollary~\ref{cor:unilateral} (i.e., unilateral coverage \citep{cui2022offline}), the sample complexity bound is $\otilde(\frac{1}{p_2 \epsilon^2})$. However, when $n>\otilde(\frac{1}{p_2})$, we already identify $(a_1,b_2)$, $(a_1,b_3)$, $(a_2,b_1)$, and $(a_3,b_1)$ as suboptimal actions with high probability. On this event, Theorem~\ref{thm:multi_bound} shows that we only suffer the coverage coefficient on the optimistic best response (which is $(a_1,b_1)$ itself), so that the sample complexity bound becomes $\otilde(\max\{\frac{1}{p_2}, \frac{1}{p_1\epsilon^2}\})\ll \otilde(\frac{1}{p_2 \epsilon^2})$.

\section{V-type Function Approximation}\label{sec:v_type}

A potential caveat of our approach in Section~\ref{sec:multi_player} is that we model Q-functions which take joint actions as inputs. In the tabular setting, the complexity of the full Q-function class has exponential dependence on the number of agents $m$, whereas prior results specialized to tabular settings do not suffer such a dependence. 

While it is known that jointly featurizing the actions can avoid such an exponential dependence \cite{zhong2022pessimistic} (in a way similar to how linear MDP results do not incur $|\Acal|$ dependence in the single-agent setting \cite{jin2020provably}), in this section we provide an alternative approach that directly subsumes the prior tabular results and produces the same rate (up to minor differences to be discussed). 
We propose a V-type variant algorithm of \BCEL, which directly models the state-value function $V_i^\pi$ with the help of a function class $\Gcal_i \subset (\Scal \rightarrow [0,\Vmax])$ for each player $i$. %

As before, we assume that the tuples $(s,\bfa,\bfr,s') \in \Dcal$ are generated as $(s, \bfa) \sim d_D$, $\bfr_i \sim r_i(s,\bfa)$ and $s' \sim P(\cdot|s,\bfa)$. In this section, we write $d_D = d_S \times d_A$, i.e., $(s,\bfa) \sim d_D \Leftrightarrow s \sim d_S, a\sim d_A(\cdot|s)$. We additionally assume that (1) $d_A(\bfa|s)> 0, \forall  (s,\bfa) \in \Scal \times \Acal$, \footnote{This assumption is w.l.o.g.~and just for technical convenience, so that the action importance weights are always well defined. Otherwise, we can simply ignore any policy $\pi$ where $\pi(\bfa|s)/d_A(\bfa|s)$ goes unbounded and assume maximum $\Delta_i^\pi$ for such $\pi$.} and (2) the behavior policy $d_A(\cdot|s)$ is known to the learner. We use the behavior policy to perform importance weighting on the actions to correct the mismatch between $d_A$ and $\pi$, and modify the loss function $\Lcal$ as follows:
for any function $g_i \in \Gcal_i$, define
\begin{align*}
\Lcal_i(g'_i,g_i, \pi;\Dcal) \coloneqq \frac{1}{n} \sum_{(s,\bfa,\bfr,s') \in \Dcal} \frac{\pi(\bfa|s)}{d_A(\bfa|s)}\left(g'_i(s,\bfa) - \bfr_i - \gamma g_i(s',\pi) \right)^2.
\end{align*}
Similarly as before, we compute empirical Bellman error $\Ecal_i(g_i,\pi;\Dcal) \coloneqq \Lcal_i(g_i,g_i,\pi;\Dcal) - \min_{g'_i \in \Gcal_i}\Lcal_i(g'_i,g_i,\pi;\Dcal)$ and construct version space $\Gcal_i^{\pi,\eps}=\{g_i \in \Gcal_i:\Ecal_i(g_i,\pi;\Dcal) \le \wg\}$. What is slightly different is that we set parameter $\wg$ as
\begin{align*}
\wg \coloneqq \frac{80 C_A(\pi)\Vmax^2\log\frac{|\Fcal||\ext|}{\delta}}{n}+30\epsf,
\end{align*}
where $C_A(\pi) \coloneqq \max_{s,\bfa} \frac{\pi(\bfa|s)}{d_A(\bfa|s)}$. Compared to $\epsv$ in Algorithm~\ref{alg:multi_player}, the extra $C_A(\pi)$ term comes from  importance weighting. With $\Gcal_i^{\pi,\wg}$ at hand, we define
\begin{align*}
g_i^{\pi,\max} \coloneqq \argmax_{g_i \in \Gcal_i^{\pi,\wg}} g_i(s_0), \quad 
g_i^{\pi,\min} \coloneqq \argmin_{g_i \in \Gcal_i^{\pi,\wg}} g_i(s_0).
\end{align*}
We then compute $\widehat{\mathrm{Gap}}_{\mathrm{EQ}}(\pi)$ which is an upper bound on equilibrium gap for any $\pi \in \Pi$: 
\begin{align*}
\widehat{\mathrm{Gap}}_{\mathrm{EQ}}(\pi) \coloneqq \max_{i \in [m]}\max_{\pi^\dagger \in \reseqi(\pi)}g_i^{\pi^\dagger,\max}(s_0)-g_i^{\pi,\min}(s_0).
\end{align*}
We select the policy by minimizing the estimated equilibrium gap:
\begin{align}\label{eq:output_policy}
\hatpi = \argmin_{\pi \in \Pi} \widehat{\mathrm{Gap}}_{\mathrm{EQ}}(\pi),
\end{align}
whose performance guarantee is shown as follows.
\begin{restatable}[V-type guarantee]{theorem}{vmulti}
\label{thm:v_multi_bound}
With probability at least $1-\delta$, for any $\pi \in \Pi$ and $\mathrm{EQ} \in \{\mathrm{NE},\mathrm{CE},\mathrm{CCE}\}$, the output policy $\hatpi$ from \cref{eq:output_policy} satisfies that
\begin{align*}
    \mathrm{Gap}^{\reseq}(\hatpi) \le \mathrm{Gap}^{\reseq}(\pi)+\frac{4\sqrt{\epsf}}{1-\gamma} + \max_{i \in [m]}\min_{\tpi_i \in \reseqi(\pi)} \left(\unc_{i}^{\tpi_i}+\unc_i^\pi+
     \mathrm{subopt}_i^\pi(\tpi_i)
     \right),
\end{align*}
where $\unc_i^\pi=g_i^{\pi,\max}(s_0)-g_i^{\pi,\min}(s_0)$ and $\mathrm{subopt}_i^{\tpi_i}=\max_{\pid \in \reseqi(\pi)} g_i^{\pid,\max}(s_0)-g_i^{\tpi_i,\max}(s_0)$. In addition, with probability at least $1-\delta$, for any player $i \in[m]$ and any $\pi \in \exti$, we have
\begin{align*}
    \unc_i^\pi &\le \min_{d \in \Delta(S)}\frac{1}{1-\gamma}\sqrt{\Cscr(d;d_S,\Gcal_i,\pi)}\eps_{\mathrm{apx}}+\frac{1}{1-\gamma} \sum_{s \in \Scal}(d^{\pi}\setminus d)(s) \left[\Delta g_i^{\pi}(s) - \gamma(P_i^\pi\Delta g_i^{\pi})(s)\right],
\end{align*}
where $\eps_{\mathrm{apx}}=\order\left(\Vmax\sqrt{C_A(\pi) \frac{\log \frac{|\Gcal||\ext|}{\delta}}{n}}+\sqrt{\epsf+\epsff}\right)$, $(d^\pi \setminus d)(s) \coloneqq \max(d^\pi(s)-d(s),0)$, $\Delta g^\pi_i(s) \coloneqq g_i^{\pi,\max}(s)-g_i^{\pi,\min}(s)$, and $(P_i^\pi g)(s)=\E_{\bfa \sim \pi(\cdot|s),s' \sim P(\cdot|s,\bfa)}[r_i(s,\bfa)+g(s')]$.
\end{restatable}
Similar to the results in \cref{sec:multi_player}, our bound enjoys an adaptive property and automatically selects the best policy $\tpi_i$, which achieves the trade-off between the suboptimality error $\mathrm{subopt}_i^\pi(\tpi_i)$ and the data coverage error $\unc_{i}^{\tpi_i}$. Furthermore, when the dataset $\Dcal$ satisfies the unilateral coverage assumption, we have the following corollary.
\begin{corollary} \label{cor:v-type}
For Nash equilibrium policy $\pis \in \Pi$, if there exists an unilateral coefficient $C_S(\pis)$ such that the following inequality holds
\begin{align}
\max_{i \in [m]} \max_{\pi^\dagger \in \resnei(\pis)} \Cscr(d^{\pi^\dagger};d_S,\Gcal_i,\pis) \le C_S(\pis).
\end{align}
With probability at least $1-\delta$, we have
\begin{small}
\begin{align*}
\mathrm{Gap}^{\resne}(\hatpi) \le \order\left(\frac{\Vmax\sqrt{\frac{\log \frac{|\Gcal||\ext|}{\delta}}{n}}+\sqrt{\epsf+\epsff}}{1-\gamma}\sqrt{C_A(\pis)C_S(\pis)}\right).
\end{align*}
\end{small}
\end{corollary}
Compared to Corollary~\ref{cor:unilateral}, our bound here depends logarithmically on the V-function class $\Gcal$ instead of the Q-function class $\Fcal$. In the tabular setting when we use fully expressive (and stationary) function classes, $\log|\Gcal| \approx \order(S)$ (via a simple covering argument) and thus our bound avoids the exponential dependence on $m$ (i.e., $\prod_{i=1}^m A_i$ dependence). %
In comparison, \citet{cui2022provably} established  $\otilde\left(\sqrt{H^4S^2\log(\mathcal{N}(\Pi))C(\pis)/n}\right)$ error bound for finite-horizon tabular Markov games, where $H$ is the horizon length and $\mathcal{N}(\Pi)$ roughly corresponds to our $|\ext|$. While finite-horizon and discounted results are generally incomparable, under a standard translation,\footnote{\citet{cui2022provably} assume that rewards are in $[0, 1]$, thus we treat $\Vmax = 1/(1-\gamma) = H$. When using fully expressive tabular classes, $\epsf=\epsff=0$.} our bound has the same rate $\otilde(n^{-1/2})$; $\sqrt{\log |\Gcal|} \approx \sqrt{SH}$ \footnote{In finite-horizon problems we need to use a non-stationary function class, therefore the extra $H$ factor.} which results in a better dependence on $S$ (saving a $\sqrt{S}$ factor) and a worse overall dependence on $H$ (we have $\sqrt{H^5}$). The slight downside is that Corollary~\ref{cor:v-type} measures distribution mismatch on actions and states separately (instead of doing them jointly as $C(\pi^*)$ in Corollary~\ref{cor:unilateral}), which is looser.

\section{Discussion and Conclusion}
\paragraph{Algorithms for two-player zero-sum games} For most part of this paper we consider the general case of multi-player general-sum Markov games. We discover that when our algorithm is specialized to the special case of two-player zero-sum (2p0s), it seemingly differs from another sample-efficient algorithm specifically designed for 2p0s and inspired by  \citet{jin2022power, cui2022provably}. In \cref{app:discussion}, we show that this difference is superficial, and these specialized algorithms can be subsumed as small variants of our algorithm. 

\paragraph{Conclusion and open problems} In this work, we study offline learning in Markov games. We design a framework that learn three popular equilibrium notions in a unified manner under general function approximation. The adaptive property of our framework enables us to relax and achieve significant improvement over the ``unilateral concentrability'' condition under certain situations.

One open problem is whether one can design a computational efficient algorithm for learning CE/CCE in offline Markov games, even in the tabular setting. A potential direction is to adapt the computationally efficient V-Learning algorithm~\citep{song2021can,jin2021v}---which runs no-regret learning dynamics at each state---to the offline setting, which may require new ideas.

\bibliography{example_paper}
\newpage
\appendix
\onecolumn

\section{Connection Between In-class Gap and Real Gap} \label{app:strategy_complete}
In this paper we consider ``in-class'' equilibrium gaps that are defined w.r.t.~certain deviation policy classes (Section~\ref{sec:eq}). It is also common to consider stronger notions of equilibrium gap, which we denote simply as $\mathrm{Gap}^{\eq}$, where the deviation policies are unrestricted, e.g., for NE and CCE, the deviation policies can take arbitrary policies \cite{nash1996non,aumann1974subjectivity}. 

To establish the connection between our in-class gap and the stronger notion of gap, we have the following strategy completeness assumption,
\begin{assumption}[Strategy completeness]\label{asm:policy}
For any player $i \in [m]$ and any $\eq \in \{\mathrm{NE},\cce\}$, we have
\begin{align*}
\max_{\pi^\dagger\in\reseqi(\pi)} V_i^{\pid}(s_0) \ge \max_{\pi': \pi'_{-i}=\pi_{-i}} V_i^{\pi'}(s_0) -\epspi.
\end{align*}
For $\ce$, we have
\begin{align*}
\max_{\pi^\dagger\in\rescei(\pi)} V_i^{\pid}(s_0) \ge \max_{\phi_i} V_i^{(\phi_i \diamond \pi_i) \odot \pi_{-i}}(s_0) -\epspi.
\end{align*}
\end{assumption}
\cref{asm:policy} requires that the (unrestricted) best-response policy is contained in $\Pi$ (and its counterpart for CE contained in $\Phi$, respectively). Under \cref{asm:policy}, it is clear that for any $\eq \in \{\mathrm{NE},\ce,\cce\}$ and $\pi$, $\mathrm{Gap}^{\eq}(\pi) \le \mathrm{Gap}^{\reseq}(\pi)+\epspi$. 
\section{A Connection in 2-player-0-sum Games}
\label{app:discussion}

For the most part of this paper, we have considered the general case of multi-player general-sum Markov games. When we are in a specialized setting, such as two-player zero-sum Markov games (2p0s), it is often the case that we can exploit the special structure and come up with alternative algorithms \cite{yan2022model,jin2022power}. %

In particular, \citet[Section 3]{cui2022provably} design an offline 2p0s algorithm for the tabular setting, and extending their algorithm to the function approximation setting (using uncertainty quantification techniques from our paper) results in an algorithm that seemingly looks very different from our \cref{alg:multi_player}. However, below we show that despite the superficial difference, the two algorithms are actually quite similar and can be derived using optimism/pessimism in the same way as in our \cref{alg:multi_player}, with only one minor difference of minimizing the duality gap versus our $\gap^{{\rm NE}}$. Consequently, for their algorithm, we can give guarantees similar to our \cref{thm:multi_bound}, by slightly adapting our algorithm and analysis.

\paragraph{2p0s setup} We now introduce some notation specialized to 2p0s games. We consider two players, where $x$-player aims to maximize the total reward while $y$-player aims to minimize it. The policy sets for $x$-player and $y$-player are denoted as $\pimax$ and $\pimin$ respectively. We consider the policy payoff $V \in [-1,1]^{|\pimax| \times |\pimin|}$, where $V^{\mu,\nu}$ denotes the utility/loss for $x$-player/$y$-player when they follow policy $\mu$ and policy $\nu$ respectively. We use $\ov$ and $\uv$ to denote the UCB and the LCB estimation of $V$ respectively. To connect these symbols with those in the main text, $V^{\mu, \nu}$ is essentially $V_{1}^{\pi}(s_0) (= - V_2^\pi(s_0))$ for $\pi = \mu \times \nu$, assuming player $1$ is the max player $x$ and player $2$ is the min player $y$. Furthermore, we have $\ov^{\mu,\nu} = \ov_1^{\pi}(s_0) = - \uv_2^{\pi}(s_0)$ and $\uv^{\mu, \nu} = \uv_1^{\pi}(s_0) = - \ov_2^\pi(s_0)$ due to the 0-sum nature of the game. 

\paragraph{Duality gap} For 2p0s game, a common learning objective is the duality gap, which is defined as:
\begin{align*}
\dual(\mu,\nu)=\max_{\mudag \in \pimax}V^{\mudag,\nu}-\min_{\nudag \in \pimin}V^{\mu,\nudag}.
\end{align*}
Since $\mudag$ and $\nudag$ can be chosen as $\mu$ and $\nu$, the duality gap is always non-negative. It measures how close the policy is to NE policy and NE policy always has zero duality gap. %
Inspired by the tabular 2p0s algorithm from \citet{cui2022provably}, one can design an offline algorithm that selects the two policies \emph{independently} with adversarial opponent under pessimistic estimation: 
\begin{align}\label{eq:two_player}
\mu=\argmax_{\mu}\min_{\nudag \in \pimin}\uv^{\mu,\nudag}  ~~ \textrm{and}  ~~ \nu=\argmin_{\nu}\max_{\mudag \in \pimax}\ov^{\mudag,\nu}. 
\end{align}
Similar ideas can also be found in \citet{jin2022power}, who design online algorithms for 2p0s games. By flipping their optimism (for online) to pessimism (for offline), we can similarly arrive at \cref{eq:two_player}.  

\paragraph{Recover \cref{eq:two_player} in our algorithmic framework} 
\cref{eq:two_player} looks very different from our \cref{alg:multi_player} at the first glance, as \cref{eq:two_player} chooses the players' policies independently whereas our \cref{alg:multi_player} requires joint optimization. We now show, however, that it is simply a minor variant of our algorithm, for which our analysis and guarantees straightforwardly extend. 

First, note that the duality gap is not the same as our objective $\mathrm{Gap}^{\reseq}(\pi)$ when specialized to 2p0s games. Recall that
$$
\mathrm{Gap}^{\reseq}(\pi) \coloneqq \max_{i \in [m]} \max_{\pi^\dagger\in \reseqi(\pi)} V_i^{\pi^\dagger}(s_0)-V_i^\pi(s_0).
$$
To recover duality gap, we can simply replace the $\max_i$ in the above objective with $\sum_i$, and obtain the following in the 2p0s case:
\begin{align*}
\sum_{i \in [m]} \max_{\pi^\dagger\in \reseqi(\pi)} V_i^{\pi^\dagger}(s_0)-V_i^\pi(s_0) = &~ (\max_{\mudag \in \pimax}\ov^{\mudag,\nu} - V^{\mu, \nu}) + (V^{\mu, \nu} - \min_{\nudag \in \pimin} \uv^{\mu,\nudag}) \\
= &~  \max_{\mudag \in \pimax}V^{\mudag,\nu}-\min_{\nudag \in \pimin}V^{\mu,\nudag} = \dual(\mu,\nu).
\end{align*}
From the above equation, we can see that our objective in \cref{alg:multi_player} is almost the same as the duality gap, up to a multiplicative factor of $2$, as for non-negative $a, b$ we have $\max(a, b) \le a+ b \le 2\max(a,b)$. Therefore, our \cref{alg:multi_player} directly enjoys duality-gap guarantees. 

However, remember that our goal here is to recover \cref{eq:two_player}, so we choose to directly work with the duality gap and relax it in the same spirit as in our \cref{alg:multi_player}: since $V \le \ov$ and $-V \le -\uv$, we have
\begin{align} \label{eq:two_player_obj}
\dual(\mu,\nu) = \max_{\mudag \in \pimax}V^{\mudag,\nu}-\min_{\nudag \in \pimin}V^{\mu,\nudag} \le  
\max_{\mudag \in \pimax}\ov^{\mudag,\nu} - \min_{\nudag \in \pimin} \uv^{\mu,\nudag}.
\end{align}
Now, \cref{eq:two_player} is recovered by noticing that $\mu$ and $\nu$ can be optimized independently on the RHS of \cref{eq:two_player_obj} %
and the optima are exactly \cref{eq:two_player}.

%
%
%
%
%

We also provide a guarantee for the above algorithm: 
\begin{proposition}
Consider a two-player zero-sum Markov game with policy payoff $V \in [-1,1]^{|\pimax| \times |\pimin|}$, let 
\begin{align*}
J(\mu,\nu) =\max_{\mudag \in \pimax}\ov^{\mudag,\nu} - \min_{\nudag \in \pimin} \uv^{\mu,\nudag}. %
\end{align*}
Let $\hatmu,\hatnu=\argmin J(\mu,\nu)$, with high probability, we have
\begin{align*}
\dual(\muhat,\nuhat) \le \min_{\tmu,\tnu \in \pimax \times \pimin} \Delta^{\tmu,\nus}+\Delta^{\mus,\tnu}+\mathrm{subopt}^{\pis}(\tmu)+\mathrm{subopt}^{\pis}(\tnu),
\end{align*}
where $\Delta^{\mu,\nu}:=\ov^{\mu,\nu}-\uv^{\mu,\nu}$, $\mathrm{subopt}^{\pis}(\tmu):=\max_{\mudag \in \pimax}\ov^{\mudag,\nus}-\ov^{\tmu,\nus}$ and $\mathrm{subopt}^{\pis}(\tnu):=\uv^{\mus,\tnu}-\min_{\nudag \in \pimin} \uv^{\mus,\nudag}$.
\end{proposition}
\begin{proof}
By standard concentration analysis, we guarantee that with high probability, $\ov^{\mu,\nu} \ge V^{\mu,\nu}$ and $\uv^{\mu,\nu} \le V^{\mu,\nu}$ hold for any $\mu,\nu  \in \pimax \times \pimin$. This implies that for any $\mu,\nu \in \pimax \times \pimin$, $\dual(\mu,\nu) \le J(\mu,\nu)$. 
For Nash policy $\pis=(\mus,\nus)$, let $\mudag=\argmax_{\mudag \in \pimax}\ov^{\mudag,\nus}$ and $\nudag=\argmin_{\nudag \in \pimin}\uv^{\mus,\nudag}$. We have
\begin{align}
J(\mus,\nus)-\dual(\mus,\nus)&=\ov^{\mudag,\nus} -\max_{\mu \in \pimax} V^{\mu,\nus} + \min_{\nu \in \pimin}V^{\mus,\nu} - \uv^{\mus,\nudag} \nonumber \\
&\le (\ov^{\tmu,\nus}-V^{\tmu,\nus}) + (V^{\mus,\tnu}-\uv^{\mus,\tnu})+\mathrm{subopt}^{\pis}(\tmu)+\mathrm{subopt}^{\pis}(\tnu)\nonumber \\
&\le \Delta^{\tmu,\nus}+\Delta^{\mus,\tnu}+\mathrm{subopt}^{\pis}(\tmu)+\mathrm{subopt}^{\pis}(\tnu) \label{eq:duality_gap_bound},
\end{align}
where $\tmu$ and $\tnu$ are arbitrary polices from $\pimax$ and $\pimin$ respectively. By the optimality of $(\hatmu,\hatnu)$ and \cref{eq:duality_gap_bound}, we obtain
\begin{align*}
\dual(\widehat \mu,\widehat \nu) \le J(\muhat, \nuhat) \le J(\mus,\nus) \le \min_{\tmu,\tnu \in \pimax \times \pimin} \Delta^{\tmu,\nus}+\Delta^{\mus,\tnu}+\mathrm{subopt}^{\pis}(\tmu)+\mathrm{subopt}^{\pis}(\tnu).
\end{align*}
The proof is completed.
\end{proof}

\section{Proofs for \cref{sec:multi_player}}
In this section, we prove \cref{thm:multi_bound}. We first show some concentration results.
\begin{lemma}
\label{lem:bernstein_general}
With probability at least $1-\delta$, for any player $i \in [m]$, any $f_i, g_1, g_2 \in \Fcal_i$, and any $\pi \in \exti$, we have
\begin{align*}
&~ \bigg| \| g_1 - \tpii f_i\|_{2,d_D}^2 - \left\| g_2 - \tpii f_i\right\|_{2,d_D}^2
\\
&~ - \frac{1}{n} \sum_{(s,\bfa,\bfr,s') \in \Dcal} \left(g_1(s,\bfa) - \bfr_i - \gamma f_i(s',\pi) \right)^2 + \frac{1}{n} \sum_{(s,\bfa,\bfr,s') \in \Dcal} \left(g_2(s,\bfa)  - \bfr_i - \gamma f_i(s',\pi) \right)^2 \bigg|
\\
\leq &~ 2 \Vmax \|g_1 - g_2\|_{2,d_D}\sqrt{\frac{ \log \frac{|\Fcal||\ext|}{\delta}}{n}} + \frac{\Vmax^2 \log \frac{|\Fcal||\ext|}{\delta}}{n}.
\end{align*}
\end{lemma}
\begin{proof}
For player $i$, we observe that 
\begin{align}
&~ \|g_1-\tpii f_i\|_{2,d_D}^2-\|g_2-\tpii f_i\|_{2,d_D}^2 \nonumber\\
= &~ \E_{d_D}\Mp{\Sp{g_1(s,\bfa)-(\tpii f_i)(s,\bfa)}^2}-\E_{d_D}\Mp{\Sp{g_2(s,\bfa)-(\tpii f_i)(s,\bfa)}^2} \nonumber\\
= &~ \E_{d_D}\Mp{(g_1(s,\bfa)-g_2(s,\bfa))(g_1(s,\bfa)+g_2(s,\bfa)-2(\tpii f)(s,\bfa))} \nonumber\\
= &~ \E_{d_D}\Mp{(g_1(s,\bfa)-g_2(s,\bfa))\E_{s' \sim P(\cdot|s,\bfa)}\Mp{g_1(s,\bfa)+g_2(s,\bfa)-2\bfr_i-2\gamma f(s',\pi)|s,\bfa}} \nonumber\\
= &~ \E_{d_D \times P}\Mp{(g_1(s,\bfa)-\bfr_i-\gamma f_i(s',\pi))^2}-\E_{d_D \times P}\Mp{(g_2(s,\bfa)-\bfr_i-\gamma f_i(s',\pi))^2}. \label{eq:expected_diff_bellman_error}
\end{align}
Let random variable $X=(g_1(s,\bfa)-\bfr_i-\gamma f_i(s',\pi))^2-(g_2(s,\bfa)-\bfr_i-\gamma f_i(s',\pi))^2$, $X$ is drawn from $d_D \times P$. We know that $\E_{d_D \times P}[X]=\E_{d_D \times P}\Mp{(g_1(s,\bfa)-\bfr_i-\gamma f_i(s',\pi))^2}-\E_{d_D \times P}\Mp{(g_2(s,\bfa)-\bfr_i-\gamma f_i(s',\pi))^2}$. For the variance, we have
\begin{align*}
    \V_{d_D \times P}\Mp{X} &\le \E_{d_D \times P}[X^2]  \\
    &\le \E_{d_D \times P}[(g_1(s,\bfa)-g_2(s,\bfa)^2(g_1(s,\bfa)+g_2(s,\bfa)-2\bfr_i-2\gamma f_i(s',\pi))^2] \\
    &\le 4\Vmax^2 \E_{d_D} \Mp{(g_1(s,\bfa)-g_2(s,\bfa))^2}.
\end{align*}
We proceed as follows
\begin{align*}
&~ \bigg| \| g_1 - \tpii f_i\|_{2,d_D}^2 - \left\| g_2 - \tpii f_i\right\|_{2,d_D}^2
\\
&~ - \frac{1}{n} \sum_{(s,\bfa,\bfr,s') \in \Dcal} \left(g_1(s,\bfa) - \bfr_i - \gamma f_i(s',\pi) \right)^2 + \frac{1}{n} \sum_{(s,\bfa,\bfr,s') \in \Dcal} \left(g_2(s,\bfa)  - \bfr_i - \gamma f_i(s',\pi) \right)^2 \bigg| \\
&= \bigg|\E_{d_D \times P}\Mp{(g_1(s,\bfa)-\bfr_i-\gamma f_i(s',\pi))^2}-\E_{d_D \times P}\Mp{(g_2(s,\bfa)-\bfr_i-\gamma f_i(s',\pi))^2} -\frac{1}{n} \sum_{j=1}^n X_j\bigg| \tag{By \cref{eq:expected_diff_bellman_error} and definition of $X$} \\
&\le \sqrt{\frac{4\Vmax^2\|g_1-g_2\|_{d_D}^2 \log \frac{|\Fcal||\ext|}{\delta}}{n}}+\frac{\Vmax^2 \log \frac{|\Fcal||\ext|}{\delta}}{n}. \tag{By Freedman's inequality}
\end{align*}
Taking a union bound over $i \in [m]$ finishes the proof.
\end{proof}
For any player $i \in [m]$ and $\pi \in \exti$, let us define
\begin{align}
\fpii& \coloneqq \argmin_{f \in \Fcal_i}\sup_{\text{admissible } d} \left\|f - \tpii f\right\|_{2,d}^2 \label{eq:def_fpii}\\
\gpii & \coloneqq  \argmin_{g \in \Fcal_i} \frac{1}{n} \sum_{(s,\bfa,\bfr,s') \in \Dcal} \left(g(s,\bfa) - \bfr_i - \gamma \fpii(s',\pi) \right)^2. \label{eq:def_gpii}
\end{align}
We bound $\|\fpii - \gpii\|_{2,d_D}$ as follows.

\begin{lemma}
\label{lem:diff_between_best_and_empirical}
Let $\fpii$ and $\gpii$ be defined as in \cref{eq:def_fpii,eq:def_gpii}. Under the success event of \cref{lem:bernstein_general}, for any player $i \in [m]$ and $\pi \in \exti$, we have
\begin{align*}
\|\fpii - \gpii\|_{2,d_D} \leq 6 \Vmax \sqrt{\frac{ \log \frac{|\Fcal||\ext|}{\delta}}{n}} + 2\sqrt{\eps_{\Fcal}}.
\end{align*}
\end{lemma}
\begin{proof}
We know that 
\begin{align}
&~ \|\fpii - \gpii\|_{2,d_D}^2 \nonumber
\\
\leq &~ 2 \|\fpii - \tpii \fpii\|_{2,d_D}^2 + 2 \|\gpii - \tpii \fpii\|_{2,d_D}^2 \nonumber
\\
= &~ 2 \|\gpii - \tpii \fpii\|_{2,d_D}^2 - 2 \|\fpii - \tpii \fpii\|_{2,d_D}^2 + 4 \|\fpii - \tpii \fpii\|_{2,d_D}^2 \nonumber
\\
\leq &~ 2 \|\gpii - \tpii \fpii\|_{2,d_D}^2 - 2 \|\fpii - \tpii \fpii\|_{2,d_D}^2 + 4 \eps_{\Fcal} \tag{By Assumption~\ref{asm:realizability}} \nonumber
\\
\overset{\text{(a)}}{\leq} &~ 4 \Vmax \sqrt{\frac{\|\gpii-\fpii\|_{2,d_D}^2 \log \frac{|\Fcal||\ext|}{\delta}}{n}} + \frac{2 \Vmax^2 \log \frac{|\Fcal||\ext|}{\delta}}{n} + 4 \eps_{\Fcal}, \label{eq:diff_fpi_g_under_mu}
\end{align}
where (a) is from
\begin{align}
&~ \|\gpii - \tpii \fpii\|_{2,d_D}^2 - \|\fpii - \tpii \fpii\|_{2,d_D}^2 \nonumber
\\
\leq &~ \frac{1}{n} \sum_{(s,\bfa,\bfr,s') \in \Dcal} \left(\gpii(s,\bfa)  - \bfr_i - \gamma \fpii(s',\pi) \right)^2 - \frac{1}{n} \sum_{(s,\bfa,\bfr,s') \in \Dcal} \left(\fpii(s,\bfa) - \bfr_i - \gamma \fpii(s',\pi) \right)^2 \nonumber
\\
&~ + 2 \Vmax \sqrt{\frac{\|\gpii-\fpii\|_{2,d_D}^2 \log \frac{|\Fcal||\ext|}{\delta}}{n}} + \frac{\Vmax^2 \log \frac{|\Fcal||\ext|}{\delta}}{n} \tag{by Lemma~\ref{lem:bernstein_general}} \nonumber
\\
\leq &~ 2 \Vmax \sqrt{\frac{\|\gpii-\fpii\|_{2,d_D}^2 \log \frac{|\Fcal||\ext|}{\delta}}{n}} + \frac{\Vmax^2 \log \frac{|\Fcal||\ext|}{\delta}}{n} \nonumber \tag{by the optimality of $g$} 
\end{align}
Solving \cref{eq:diff_fpi_g_under_mu} finishes the proof.
\end{proof}

In the following lemma, we show that the best approximation of $Q_i^\pi$ is contained in $\Fcal_i^{\pi,\epsv}$.
\begin{lemma}
\label{lem:best_approximation_empirical_error}
Under the success event of \cref{lem:bernstein_general}, for any player $i \in [m]$ and $\pi \in \exti$, the following inequality for $\Ecal_i(\fpii,\pi;\Dcal)$ holds
\begin{align*}
\Ecal_i(\fpii,\pi;\Dcal) \leq \frac{80 \Vmax^2 \log \frac{|\Fcal||\ext|}{\delta}}{n} + 30 \eps_{\Fcal} \eqqcolon \epsv.
\end{align*}
\end{lemma}
\begin{proof}
Applying \cref{lem:bernstein_general} and \cref{lem:diff_between_best_and_empirical}, we obtain that
\begin{align}
&~ \bigg| \left\|\fpii - \tpii \fpii  \right\|_{2,d_D}^2 - \left\|\gpii - \tpii \fpii  \right\|_{2,d_D}^2 - \nonumber
\\
&~ \frac{1}{n} \sum_{(s,\bfa,\bfr,s') \in \Dcal} \left(\fpii(s,\bfa) - \bfr_i - \gamma \fpii(s',\pi) \right)^2 + \frac{1}{n} \sum_{(s,\bfa,\bfr,s') \in \Dcal} \left(\gpii(s,\bfa)  - \bfr_i - \gamma \fpii(s',\pi) \right)^2 \bigg| \nonumber
\\
\leq &~ 2 \Vmax \|\fpii - \gpii\|_{2,d_D}\sqrt{\frac{ \log \frac{|\Fcal||\ext|}{\delta}}{n}} + \frac{\Vmax^2 \log \frac{|\Fcal||\ext|}{\delta}}{n}\nonumber \\
\label{eq:fpi_g_expected_empirical}
\leq &~ 4 \Vmax \sqrt{\frac{ \log \frac{|\Fcal||\ext|}{\delta}}{n} \eps_{\Fcal}} + \frac{13 \Vmax^2 \log \frac{|\Fcal||\ext|}{\delta}}{n}. 
\end{align}
Then, we bound $\|\fpii - \tpii \fpii  \|_{2,d_D}^2 - \|\gpii - \tpii \fpii  \|_{2,d_D}^2$ as follows,
\begin{align}
&~ \left\|\fpii - \tpii \fpii  \right\|_{2,d_D}^2 - \left\|\gpii - \tpii \fpii  \right\|_{2,d_D}^2 \nonumber
\\
\leq &~ \left( \left\|\fpii - \tpii \fpii  \right\|_{2,d_D} + \left\|\gpii - \tpii \fpii  \right\|_{2,d_D}\right) \left| \left\|\fpii - \tpii \fpii  \right\|_{2,d_D} - \left\|\gpii - \tpii \fpii  \right\|_{2,d_D} \right| \nonumber
\\
\leq &~ \left( 2 \left\|\fpii - \tpii \fpii  \right\|_{2,d_D} + \left\|\fpii - \gpii \right\|_{2,d_D}\right) \left\|\fpii - \gpii \right\|_{2,d_D} \tag{By triangle inequality} \nonumber
\\
\leq &~ 36 \Vmax \sqrt{\frac{ \log \frac{|\Fcal||\ext|}{\delta}}{n} \eps_{\Fcal}} + 36 \Vmax^2  \frac{ \log \frac{|\Fcal||\ext|}{\delta}}{n} + 8 \eps_{\Fcal}. \tag{By \cref{asm:realizability} and \cref{lem:diff_between_best_and_empirical}}\label{eq: diff_fpi_tcalfpi_g_tcalfpi}
\end{align}
Combining this with \cref{eq:fpi_g_expected_empirical}, we get
\begin{align*}
&~ \frac{1}{n} \sum_{(s,\bfa,\bfr,s') \in \Dcal} \left(\fpii(s,\bfa) - \bfr_i - \gamma \fpii(s',\pi) \right)^2 - \frac{1}{n} \sum_{(s,\bfa,\bfr,s') \in \Dcal} \left(\gpii(s,\bfa)  - \bfr_i - \gamma \fpii(s',\pi) \right)^2
\\
\leq &~ \left\|\fpii - \tpii \fpii  \right\|_{2,d_D}^2 - \left\|\gpii - \tpii \fpii  \right\|_{2,d_D}^2+4 \Vmax \sqrt{\frac{ \log \frac{|\Fcal||\ext|}{\delta}}{n} \eps_{\Fcal}} + \frac{13 \Vmax^2 \log \frac{|\Fcal||\ext|}{\delta}}{n} 
\\
\leq &~ 40 \Vmax \sqrt{\frac{ \log \frac{|\Fcal||\ext|}{\delta}}{n} \eps_{\Fcal}} + \frac{59 \Vmax^2 \log \frac{|\Fcal||\ext|}{\delta}}{n} + 8\eps_{\Fcal} 
\\
\leq &~ \frac{80 \Vmax^2 \log \frac{|\Fcal||\ext|}{\delta}}{n} + 30 \eps_{\Fcal}. \tag{$\sqrt{ab} \le \frac{a+b}{2}$}
\end{align*}
\end{proof}

Then, we show that $|\fpii(s_0,\pi)-V_i^\pi(s_0)|$ is upper bounded as follows
\begin{lemma}\label{lem:fpij_and_q}
For any player $i \in [m]$ and $\pi \in \exti$, let $\fpii$ be defined as in \cref{eq:def_fpii}, we have
\begin{align*}
|\fpii(s_0,\pi)-V_i^\pi(s_0)| \le \epsone.
\end{align*}
\end{lemma}
\begin{proof}
By invoking \cref{lem:evaluation_error}, we get
\begin{align*}
|\fpii(s_0,\pi)-V_i^\pi(s_0)| &\le \frac{\left|\E_{s,\bfa \sim d^\pi}[f(s,\bfa)-(\tpii f)(s,\bfa)] \right|}{1-\gamma} \\
&\le \frac{\|f-\tpii f\|_{2,d^\pi}}{1-\gamma} \le \epsone.
\end{align*}
The second inequality is from Jensen's inequality and the last inequality follows from \cref{asm:realizability}.
\end{proof}
For the version space $\Fcal_i^{\pi,\epsv}$, we define
\begin{align*}
f_i^{\pi,\max} &\coloneqq \argmax_{f_i \in \Fcal_i^{\pi,\epsv}} f_i(s_0,\pi) \\ 
f_i^{\pi,\min} &\coloneqq \argmin_{f_i \in \Fcal_i^{\pi,\epsv}} f_i(s_0,\pi).
\end{align*}
We show that $f_i^{\pi,\max}(s_0,\pi)$ and $f_i^{\pi,\min}(s_0,\pi)$ are the upper bound and the lower bound on the value function $V_i^\pi(s_0)$ respectively.
\begin{restatable}{lemma}{minmax}\label{lem:minmax_error}
Under the success event of \cref{lem:bernstein_general}, for any player $i \in[m]$ and any $\pi \in \exti$, the following two inequalities hold
\begin{align*}
    f_i^{\pi,\max}(s_0,\pi) &\geq V_i^{\pi}(s_0) - \epsone \\
    f_i^{\pi,\min}(s_0,\pi) &\leq V_i^{\pi}(s_0) + \epsone.
\end{align*}
\end{restatable}
\begin{proof}
By \cref{lem:best_approximation_empirical_error}, we know that under the success event of \cref{lem:bernstein_general}, $\fpii \in \Fcal^{\pi,\epsv}_i$. Then, we invoke \cref{lem:fpij_and_q} and get
\begin{align*}
    \fpimin(s_0,\pi) \le \fpii(s_0,\pi) \le Q_i^\pi(s_0,\pi)+\frac{\sqrt{\epsf}}{1-\gamma} = V_i^\pi(s_0) + \frac{\sqrt{\epsf}}{1-\gamma}.
\end{align*}
Similarly, we have
\begin{align*}
    \fpimax(s_0,\pi) \ge \fpii(s_0,\pi) \ge Q_i^\pi(s_0,\pi)-\frac{\sqrt{\epsf}}{1-\gamma} = V_i^\pi(s_0) - \frac{\sqrt{\epsf}}{1-\gamma}.
\end{align*}
\end{proof}

We now show that $\Ecal_i(f_i,\pi;\Dcal)$ could effectively estimate $\|f_i - \tpii f_i\|_{2,d_D}^2$.
\begin{lemma}
\label{lem:bellman_error}
Under the success event of \cref{lem:bernstein_general}, for any player $i \in [m]$ and any $\pi \in \exti$, given $\eps>0$, if $f_i \in \Fcal_i$ satisfies that $\Ecal_i(f_i,\pi;\Dcal) \leq \eps$, we have
\begin{align*}
\| f_i - \tpii f_i\|_{2,d_D} \leq 8\Vmax \sqrt{\frac{\log \frac{|\Fcal||\ext|}{\delta}}{n}} + 4\sqrt{\eps_{\Fcal,\Fcal}} + \sqrt{\eps}.
\end{align*}
\end{lemma}
\begin{proof}
Let $\gpii$ be defined as in \cref{eq:def_gpii}, we first upper bound term $\|\gpii - \tpii f_i\|_{2,d_D}$. Let us define 
\begin{align}
\fpiid \coloneqq \argmin_{f'_i \in \Fcal_i} \left\| f'_i - \tpii f_i\right\|_{2,d_D}^2. \nonumber
\end{align}
By invoking \cref{lem:bernstein_general}, we obtain that 
\begin{align*}
&~ \bigg| \left\| \gpii - \tpii f_i\right\|_{2,d_D}^2 - \left\| \fpiid - \tpii f_i\right\|_{2,d_D}^2 - \frac{1}{n} \sum_{(s,\bfa,\bfr,s') \in \Dcal} \left(g(s,\bfa)  - \bfr_i - \gamma f_i(s',\pi) \right)^2
\\
&~ + \frac{1}{n} \sum_{(s,\bfa,\bfr,s') \in \Dcal} \left(\fpiid(s,\bfa)  - r - \gamma f_i(s',\pi) \right)^2\bigg|
\\
\leq &~ 2 \Vmax \|\gpii - \fpiid\|_{2,d_D}\sqrt{\frac{ \log \frac{|\Fcal||\ext|}{\delta}}{n}} + \frac{\Vmax^2 \log \frac{|\Fcal||\ext|}{\delta}}{n}.
\end{align*}
Rearranging the terms and we have
\begin{align}
&~ \left\| \gpii - \tpii f_i\right\|_{2,d_D}^2 \nonumber
\\
\leq &~ \frac{1}{n} \sum_{(s,\bfa,\bfr,s') \in \Dcal} \left(\gpii(s,\bfa)  - \bfr_i - \gamma f_i(s',\pi) \right)^2 - \frac{1}{n} \sum_{(s,\bfa,\bfr,s') \in \Dcal} \left(\fpiid(s,\bfa)  - \bfr_i - \gamma f_i(s',\pi) \right)^2 \nonumber
\\
&~ + \left\| \fpiid - \tpii f_i\right\|_{2,d_D}^2 + 2 \Vmax \|\gpii - \fpiid\|_{2,d_D}\sqrt{\frac{ \log \frac{|\Fcal||\ext|}{\delta}}{n}} + \frac{\Vmax^2 \log \frac{|\Fcal||\ext|}{\delta}}{n} \nonumber
\\
\leq &~ \left\| \fpiid - \tpii f_i\right\|_{2,d_D}^2 + 2 \Vmax \|\gpii - \fpiid\|_{2,d_D}\sqrt{\frac{ \log \frac{|\Fcal||\ext|}{\delta}}{n}} + \frac{\Vmax^2 \log \frac{|\Fcal||\ext|}{\delta}}{n} \nonumber
\\
\leq &~ \eps_{\Fcal,\Fcal} + 2 \Vmax \|\gpii - \tpii f_i\|_{2,d_D}\sqrt{\frac{ \log \frac{|\Fcal||\ext|}{\delta}}{n}} + 2 \Vmax \sqrt{\eps_{\Fcal,\Fcal}}\sqrt{\frac{ \log \frac{|\Fcal||\ext|}{\delta}}{n}} + \frac{\Vmax^2 \log \frac{|\Fcal||\ext|}{\delta}}{n} \nonumber
\\
\label{eq:fc2}
\leq &~ 2 \Vmax \|\gpii - \tpii f_i\|_{2,d_D}\sqrt{\frac{ \log \frac{|\Fcal||\ext|}{\delta}}{n}} + \frac{2 \Vmax^2 \log \frac{|\Fcal||\ext|}{\delta}}{n} + 2\eps_{\Fcal,\Fcal}.
\end{align}
The second inequality is from the optimality of $\gpii$. The third inequality follows from \cref{asm:completeness} and $\|\gpii-\fpiid\|_{2,d_D} \le \|\gpii-\tcalpi f_i\|_{2,d_D}+\|\fpiid-\tcalpi f_i\|_{2,d_D}$. 
The last inequality is from $\sqrt{ab} \le \frac{a+b}{2}$. By solving \cref{eq:fc2}, we get
\begin{align}\label{eq:gpii_2_norm}
 \|\gpii - \tpii f_i\|_{2,d_D} \le  3 \Vmax \sqrt{\frac{ \log \frac{|\Fcal||\ext|}{\delta}}{n}} + \sqrt{2\eps_{\Fcal,\Fcal}}.
\end{align}
Then, we invoke \cref{lem:bernstein_general} for $\fpii$
\begin{align*}
&~ \bigg| \| f_i - \tpii f_i\|_{2,d_D}^2 - \left\| \gpii - \tpii f_i\right\|_{2,d_D}^2
\\
&~ - \frac{1}{n} \sum_{(s,\bfa,\bfr,s') \in \Dcal} \left(f_i(s,\bfa) - \bfr_i - \gamma f_i(s',\pi) \right)^2 + \frac{1}{n} \sum_{(s,\bfa,\bfr,s') \in \Dcal} \left(\gpii(s,\bfa)  - \bfr_i - \gamma f_i(s',\pi) \right)^2 \bigg|
\\
\leq &~ 2 \Vmax \|f_i - \gpii\|_{2,d_D}\sqrt{\frac{ \log \frac{|\Fcal||\ext|}{\delta}}{n}} + \frac{\Vmax^2 \log \frac{|\Fcal||\ext|}{\delta}}{n}
\\
\leq &~ 2 \Vmax \left( \|f_i - \tpii f_i\|_{2,d_D} + \|\gpii - \tpii f_i\|_{2,d_D} \right)\sqrt{\frac{ \log \frac{|\Fcal||\ext|}{\delta}}{n}} + \frac{\Vmax^2 \log \frac{|\Fcal||\ext|}{\delta}}{n}
\\
\leq &~ 2 \Vmax \|f_i - \tpii f_i\|_{2,d_D} \sqrt{\frac{ \log \frac{|\Fcal||\ext|}{\delta}}{n}} + 3 \Vmax \sqrt{\frac{ \log \frac{|\Fcal||\ext|}{\delta}}{n} \varepsilon_{\Fcal,\Fcal}} + \frac{7 \Vmax^2 \log \frac{|\Fcal||\ext|}{\delta}}{n}. \tag{By \cref{eq:gpii_2_norm}}
\end{align*}
Rearranging the terms, we get
\begin{align}
&~ \| f_i - \tpii f_i\|_{2,d_D}^2 \nonumber
\\
\leq &~ \left\| \gpii - \tpii f_i\right\|_{2,d_D}^2 + \frac{1}{n} \sum_{(s,\bfa,\bfr,s') \in \Dcal} \left(f_i(s,\bfa) - \bfr_i - \gamma f_i(s',\pi) \right)^2 - \frac{1}{n} \sum_{(s,\bfa,\bfr,s') \in \Dcal} \left(\gpii(s,\bfa)  - \bfr_i - \gamma f_i(s',\pi) \right)^2 \nonumber
\\
&~ + 2 \Vmax \|f_i - \tpii f_i\|_{2,d_D} \sqrt{\frac{ \log \frac{|\Fcal||\ext|}{\delta}}{n}} + 3 \Vmax \sqrt{\frac{ \log \frac{|\Fcal||\ext|}{\delta}}{n} \varepsilon_{\Fcal,\Fcal}} + \frac{7 \Vmax^2 \log \frac{|\Fcal||\ext|}{\delta}}{n} \nonumber
\\
\leq &~ \left( 3 \Vmax \sqrt{\frac{ \log \frac{|\Fcal||\ext|}{\delta}}{n}} + \sqrt{2\eps_{\Fcal,\Fcal}} \right)^2 + \varepsilon \tag{By \cref{eq:gpii_2_norm} and $\Ecal(f,\pi;\Dcal) \leq \varepsilon$} \nonumber
\\
&~ + 2 \Vmax \|f_i - \tpii f_i\|_{2,d_D} \sqrt{\frac{ \log \frac{|\Fcal||\ext|}{\delta}}{n}} + 3 \Vmax \sqrt{\frac{ \log \frac{|\Fcal||\ext|}{\delta}}{n} \varepsilon_{\Fcal,\Fcal}} + \frac{7 \Vmax^2 \log \frac{|\Fcal||\ext|}{\delta}}{n} \nonumber
\\
= &~ 2 \Vmax \|f_i - \tpii f_i\|_{2,d_D} \sqrt{\frac{ \log \frac{|\Fcal||\ext|}{\delta}}{n}} + 12 \Vmax \sqrt{\frac{ \log \frac{|\Fcal||\ext|}{\delta}}{n} \varepsilon_{\Fcal,\Fcal}} + \frac{16 \Vmax^2 \log \frac{|\Fcal||\ext|}{\delta}}{n} + 2\varepsilon_{\Fcal,\Fcal} + \varepsilon. \label{eq:bellman_error_f_mu}
\end{align}
Solving \cref{eq:bellman_error_f_mu} and using AM-GM inequality finishes the proof.
\end{proof}
We upper bound $\unc_i^\pi$ as follows.
\upperdelta*
\begin{proof}
We apply \cref{lem:evaluation_error} for $\fpimax$ and $\fpimin$ and obtain
\begin{align*}
    &~ \fpimax(s_0,\pi)-\fpimin(s_0,\pi) \\
    = &~ \fpimax(s_0,\pi)-V_i^{\pi}(s_0)+V_i^{\pi}(s_0)-\fpimin(s_0,\pi). \\
    = &~ \frac{1}{1 - \gamma} \left( \Ebb_{d_{\pi}} \left[\fpimax - \tpii \fpimax \right] - \Ebb_{d_{\pi}} \left[\fpimin - \tpii \fpimin \right] \right) \tag{By Lemma~\ref{lem:evaluation_error}}
    \\
    = &~ \frac{1}{1 - \gamma} \big( \Ebb_{d} \left[ \left(\fpimax - \tpii \fpimax \right) - \left(\fpimin - \tpii \fpimin \right)\right]\\
    &~ + \Ebb_{d_\pi} \left[(\fpimax - \tpii \fpimax) - (\fpimin - \tpii \fpimin) \right]\\
    &~ - \Ebb_{d} \left[(\fpimax - \tpii \fpimax) - (\fpimin - \tpii \fpimin) \right] \big) \\
    = &~ \frac{1}{1 - \gamma} \underbrace{\Ebb_{d} \left[ \left(\fpimax - \tpii \fpimax \right) - \left(\fpimin - \tpii \fpimin \right)\right]}_{\text{(I)}} \\
    &~ + \frac{1}{1 - \gamma} \underbrace{\left(\Ebb_{d_\pi} \left[\Delta \fpii - \gamma P^{\pi} \Delta \fpii \right]  - \Ebb_{d} \left[\Delta \fpii - \gamma P^{\pi} \Delta \fpii \right] \right)}_{\text{(II)}}, \tag{$\Delta \fpii \coloneqq \fpimax - \fpimin$}
\end{align*}
where $d \in \Delta(\Scal \times \Acal)$ is an arbitrary distribution.
For the term (I), we have
\begin{align*}
\text{(I)} \leq &~ \left| \Ebb_{d} \left[ \left(\fpimax - \tpii \fpimax \right)\right] \right| + \left| \Ebb_{d} \left[ \left(\fpimin - \tpii \fpimin \right)\right] \right|
\\
\leq &~ \|\fpimax - \tpii \fpimax\|_{2,d} + \|\fpimin - \tpii \fpimin\|_{2,d} \tag{By Jensen's inequality}\\
\leq &~ \sqrt{\Cscr(d;d_D,\Fcal_i,\pi)}\left( \|\fpimax - \tpii \fpimax\|_{2,d_D} + \|\fpimin - \tpii \fpimin\|_{2,d_D}\right).
\end{align*}
Recall that $f_i^{\pi,\max} \coloneqq \argmax_{f_i \in \Fcal_i^{\pi,\epsv}} f_i(s_0,\pi)$ and $f_i^{\pi,\min} \coloneqq \argmin_{f_i \in \Fcal_i^{\pi,\epsv}} f_i(s_0,\pi)$ and $\epsv=\frac{80\Vmax^2\log\frac{|\Fcal||\ext|}{\delta}}{n}+30\epsf$. We invoke \cref{lem:bellman_error} and have
\begin{align}
\text{(I)} \leq \sqrt{\Cscr(d;d_D,\Fcal_i,\pi)}\order\left(\Vmax\sqrt{\frac{\log \frac{|\Fcal||\ext|}{\delta}}{n}}+\sqrt{\epsf+\epsff}\right). \label{eq:term_one}
\end{align}
For term (II), we have
\begin{align}
\text{(II)} \leq &~ \sum_{(s,\bfa) \in \Scal\times \Acal} (d_{\pi}\setminus d)(s,\bfa) \left[\Delta \fpii(s,\bfa) - \gamma(P^{\pi}\Delta \fpii)(s,\bfa)\right]
\nonumber \\
&~ + \sum_{(s,\bfa) \in \Scal\times \Acal} \mathbb{I}(d(s,\bfa) > d_{\pi}(s,\bfa)) \left[d(s,\bfa) - d_{\pi}(s,\bfa)\right] \left|\Delta \fpii(s,\bfa) - \gamma(P^{\pi}\Delta \fpii)(s,\bfa)\right|
\nonumber \\
\leq &~ \sum_{(s,\bfa) \in \Scal\times \Acal} (d_{\pi}\setminus d)(s,\bfa) \left[\Delta \fpii(s,\bfa) - \gamma(P^{\pi}\Delta \fpii)(s,\bfa)\right] \nonumber \\
&~ + \Ebb_{d} \left[ \left|\fpimax - \tpii \fpimax \right| + \left|\fpimin - \tpii \fpimin \right|\right]
\nonumber \\
\leq &~\sum_{(s,\bfa) \in  \Scal\times \Acal} (d_{\pi}\setminus d)(s,\bfa) \left[\Delta \fpii(s,\bfa) - \gamma(P^{\pi}\Delta \fpii)(s,\bfa)\right] \nonumber \\
&+\sqrt{\Cscr(d;d_D,\Fcal_i,\pi)}\order\left(\Vmax\sqrt{\frac{\log \frac{|\Fcal||\ext|}{\delta}}{n}}+\sqrt{\epsf+\epsff}\right).\label{eq:term_two}
\end{align}
The last step is from the analysis of term (I). Combining \cref{eq:term_one} and \cref{eq:term_two}, we get
\begin{align*}
    \fpimax(s_0,\pi)-\fpimin(s_0,\pi) &\leq \min_d\frac{1}{1-\gamma}\sqrt{\Cscr(d;d_D,\Fcal_i,\pi)} \order\left(\Vmax\sqrt{\frac{\log \frac{|\Fcal||\ext|}{\delta}}{n}}+\sqrt{\epsf+\epsff}\right) \\
    &+\frac{1}{1-\gamma}\sum_{(s,\bfa) \in \Scal\times \Acal}(d_{\pi}\setminus d)(s,\bfa) \left[\Delta \fpii(s,\bfa) - \gamma(P^{\pi}\Delta \fpii)(s,\bfa)\right].
\end{align*}
The proof is completed.
\end{proof}

Then, we show that $\mathrm{Gap}^{\reseq}(\pi)$ is upper bounded by the estimated gap $\estgap(\pi)$.

\begin{lemma}\label{lem:eq_upper_gap}
Under the success event of \cref{lem:bernstein_general}, for any $\pi \in \Pi$, we have
\begin{align*}
     \mathrm{Gap}^{\reseq}(\pi) &\le \estgap(\pi) + \frac{2\sqrt{\epsf}}{1-\gamma}.
\end{align*}
\end{lemma}
\begin{proof}
Let $\pidi=\argmax_{\pi^\dagger \in \resi(\pi)} V_i^{\pi^\dagger}(s_0)$. With \cref{lem:minmax_error}, for any player $i \in [m]$, we have with probability at least $1-\delta$,
\begin{align*}
V_i^{\pidi}(s_0) \le \max_{f_i \in \Fcal_i^{\pidi,\epsv}}f_i(s_0,\pidi)+\epsone.
\end{align*}
Recall the definition of $\mathrm{Gap}^{\reseq}(\pi)$, we obtain
\begin{align*}
\mathrm{Gap}^{\reseq}(\pi)&=\max_{i \in [m]}\max_{\pi^\dagger \in \resi(\pi)}V_i^{\pi^\dagger}(s_0)-V_i^\pi(s_0) \\
&\le \max_{i \in [m]} \left(\max_{f_i \in \Fcal_i^{\pidi,\epsv}}f_i(s_0,\pidi)-V_i^\pi(s_0)\right)+\epsone. \\
&\le \max_{i \in [m]} \left(\max_{\pi^\dagger \in \resi(\pi)}
\ov_i^{\pi^\dagger}(s_0)-V_i^\pi(s_0)\right)+\epsone \\ \tag{By definition of $\ov_i^{\pi^\dagger}(s_0)$} \\
&\le \max_{i \in [m]} \left(\max_{\pi^\dagger \in \resi(\pi)}
\ov_i^{\pi^\dagger}(s_0)-\uv_i^\pi(s_0)\right)+\frac{2\sqrt{\epsf}}{1-\gamma} \tag{By definition of $\uv_i^\pi(s_0)$ and \cref{lem:minmax_error}} \\
&=\estgap(\pi)+\frac{2\sqrt{\epsf}}{1-\gamma}.
\end{align*}
\end{proof}

Now we are ready to prove \cref{thm:multi_bound}.
\multi*
\begin{proof}
Let $\pidi=\argmax_{\pi^\dagger \in \resi(\pi)} \ov_i^{\pi^\dagger}(s_0)$. With probability at least $1-\delta$, for each player $i \in[m]$, we upper bound $\ov_i^{\pidi}(s_0)-\uv_i^\pi(s_0)$ as
\begin{align*}
\ov_i^{\pidi}(s_0)-\uv_i^\pi(s_0)&=\ov_i^{\tpi_i}(s_0)-\uv_i^\pi(s_0)+\mathrm{subopt}_i^\pi(\tpi_i) \tag{$\tpi_i$ is an arbitrary policy from $\exti(\pi)$}\\
&=f_i^{\tpi_i,\max}(s_0)-f_i^{\pi,\min}(s_0)+\mathrm{subopt}_i^\pi(\tpi_i) \\
&\le V_i^{\tpi_i}(s_0)+\unc_i^{\tpi_i}-\fpimax(s_0)+\unc_i^{\pi}+\frac{\sqrt{\epsf}}{1-\gamma}+\mathrm{subopt}_i^\pi(\tpi_i) \tag{By definition of $\unc_i^{\tpi_i}$ and \cref{lem:minmax_error}} \\
&\le V_i^{\tpi_i}(s_0)-V_i^\pi(s_0)+\unc_i^{\tpi_i}+\unc_i^{\pi}+\frac{2\sqrt{\epsf}}{1-\gamma}+\mathrm{subopt}_i^\pi(\tpi_i) \tag{By \cref{lem:minmax_error}} \\
&\le \max_{\pid \in \resi(\pi)} V_i^{\pid}(s_0)-V_i^{\pi}(s_0)+\unc_i^{\tpi_i}+\unc_i^{\pi}+\frac{2\sqrt{\epsf}}{1-\gamma}+\mathrm{subopt}_i^\pi(\tpi_i).
\end{align*}
This directly implies that
\begin{align}\label{eq:estgap_upper_bound}
\estgap(\pi) \le \mathrm{Gap}^{\reseq}(\pi)+\frac{2\sqrt{\epsf}}{1-\gamma}+\max_{i \in [m]}\min_{\tpi_i \in \resi(\pi)} \left(\unc_i^{\tpi_i}+\unc_i^{\pi}+\mathrm{subopt}_i^\pi(\tpi_i)\right).
\end{align}
By the optimality of $\hatpi$, for any $\pi \in \Pi$, we have
\begin{align*}
\mathrm{Gap}^{\reseq}(\hatpi) &\le \estgap(\hatpi) + \frac{2\sqrt{\epsf}}{1-\gamma} \\
&\le \estgap(\pi) + \frac{2\sqrt{\epsf}}{1-\gamma} \\
&\le \mathrm{Gap}^{\reseq}(\pi) + \frac{4\sqrt{\epsf}}{1-\gamma} + \max_{i \in [m]}\min_{\tpi_i \in \resi(\pi)} \left(\unc_i^{\tpi_i}+\unc_i^{\pi}+\mathrm{subopt}_i^\pi(\tpi_i)\right). \tag{By \cref{eq:estgap_upper_bound}}
\end{align*}
This completes the proof.
\end{proof}

\section{Proofs for \cref{sec:v_type}}
In this section, we prove \cref{thm:v_multi_bound}. We start with some concentration results.
\begin{lemma}
\label{lem:v_bernstein_general}
With probability at least $1-\delta$, for any $g_1, g_2, h \in \Gcal_i$ and $\pi \in \exti$, we have
\begin{align*}
&~ \bigg| \| g_1 - \tpii h\|_{2,d_S}^2 - \left\| g_2 - \tpii h\right\|_{2,d_S}^2
\\
&~ - \frac{1}{n} \sum_{(s,\bfa,\bfr,s') \in \Dcal} \frac{\pi(\bfa|s)}{d_A(\bfa|s)} \left(g_1(s) - \bfr_i - \gamma h(s') \right)^2 + \frac{1}{n} \sum_{(s,\bfa,\bfr,s') \in \Dcal} \frac{\pi(\bfa|s)}{d_A(\bfa|s)} \left(g_2(s)  - \bfr_i - \gamma h(s') \right)^2 \bigg|
\\
\leq &~ 2 \Vmax \|g_1 - g_2\|_{2,d_S}\sqrt{C_A(\pi) \frac{ \log \frac{|\Gcal||\ext|}{\delta}}{n}} + \frac{C_A(\pi) \Vmax^2 \log \frac{|\Gcal||\ext|}{\delta}}{n},
\end{align*}
where $C_A(\pi) \coloneqq \max_{s,\bfa} \frac{\pi(\bfa|s)}{d_A(\bfa|s)}$.
\end{lemma}
\begin{proof}
First, we observe that 
\begin{align}
&~ \|g_1-\tpii h\|_{2,d_S}^2-\|g_2-\tpii h\|_{2,d_S}^2 \nonumber\\
= &~ \E_{s \sim d_S, \bfa \sim \pi(\cdot|s), s' \sim P(\cdot|s,\bfa)}[(g_1(s)-r_i(s,\bfa)-\gamma h(s'))^2]-\E_{s \sim d_S, \bfa \sim \pi(\cdot|s), s' \sim P(\cdot|s,\bfa)}[(g_2(s)-r_i(s,\bfa)-\gamma h(s'))^2]. \nonumber
\end{align}
Let random variable $X=\frac{\pi(\bfa|s)}{d_A(\bfa|s)}(g_1(s_i)-r_i(s_i,\bfa)-\gamma h(s'_i))^2-\frac{\pi(\bfa|s)}{d_A(\bfa|s)}(g_2(s_i)-r_i(s_i,\bfa)-\gamma h(s'_i))^2$, $X$ is drawn from $d_S \times d_A \times \Pcal$. Then we obtain
\begin{align*}
&~ \bigg| \| g_1 - \tpii f\|_{2,d_S}^2 - \left\| g_2 - \tpii f\right\|_{2,d_S}^2
\\
&~ - \frac{1}{n} \sum_{(s,\bfa,\bfr,s') \in \Dcal} \frac{\pi(\bfa|s)}{d_A(\bfa|s)}\left(g_1(s) - \bfr_i - \gamma h(s') \right)^2 + \frac{1}{n} \sum_{(s,\bfa,\bfr,s') \in \Dcal} \frac{\pi(\bfa|s)}{d_A(\bfa|s)} \left(g_2(s)  - \bfr_i - \gamma h(s') \right)^2 \bigg| \\
&= \bigg|\E_{d_S \times \pi \times \Pcal}\Mp{(g_1(s)-r_i(s,\bfa)-\gamma h(s'))^2}-\E_{d_S \times 
 \pi \times \Pcal}\Mp{(g_2(s)-r_i(s,\bfa)-\gamma h(s'))^2} -\frac{1}{n} \sum_{i=1}^n X_i\bigg| \tag{By definition of $X$}
\end{align*}
Here $\E_{d_S \times d_A \times \Pcal}[X]=\E_{d_S \times \pi \times \Pcal}\Mp{(g_1(s)-r_i(s,\bfa)-\gamma h(s'))^2}-\E_{d_S \times \pi \times \Pcal}\Mp{(g_2(s)-r_i(s,\bfa)-\gamma h(s'))^2}$. For the variance, we have
\begin{align*}
    &~\V_{d_S \times d_A \times \Pcal}\Mp{X} \\
    \le&~ \E_{d_S \times d_A \times \Pcal}\left[\frac{\pi(\bfa|s)^2}{d_A(\bfa|s)^2}(g_1(s)-g_2(s))^2(g_1(s)+g_2(s)-2r_i(s,\bfa)-2\gamma h(s'))^2\right]\\
    \le&~ 4\Vmax^2 \E_{d_S \times d_A} \Mp{\frac{\pi(\bfa|s)^2}{d_A(\bfa|s)^2}(g_1(s)-g_2(s))^2} \\
    \le&~ 4\Vmax^2 \max_{s,\bfa} \frac{\pi(\bfa|s)^2}{d_A(\bfa|s)} \E_{d_S} \Mp{(g_1(s)-g_2(s))^2}
\end{align*}
Let $C_A(\pi) \coloneqq \max_{s,\bfa} \frac{\pi(\bfa|s)}{d_A(\bfa|s)}$. By Freedman's inequality and union bound, we have with probability at least $1-\delta$,
\begin{align}
    \left\|\E[X]-\frac{1}{n}\sum_{i=1}^n X_i\right\| \le \sqrt{\frac{4\Vmax^2 
 C_A(\pi)\|g_1-g_2\|_{d_S}^2 \log \frac{|\Gcal||\ext|}{\delta}}{n}}+\frac{C_A(\pi) \Vmax^2 \log \frac{|\Gcal||\ext|}{\delta}}{n}. \nonumber
\end{align}
\end{proof}
For any player $i \in [m]$ and $\pi \in \exti$, let us define
\begin{align}
\gpii& \coloneqq \argmin_{g \in \Gcal_i}\sup_{\text{admissible } d} \left\|g - \tpii g\right\|_{2,d}^2 \label{eq:v_def_fpii}\\
\hpii & \coloneqq  \argmin_{h \in \Gcal_i} \frac{1}{n} \sum_{(s,\bfa,\bfr,s') \in \Dcal} \frac{\pi(\bfa|s)}{d_A(\bfa|s)} \left(h(s) - \bfr_i - \gamma \gpii(s') \right)^2. \label{eq:v_def_gpii}
\end{align}
We bound $\|\gpii - \hpii\|_{2,d_S}$ as follows.
\begin{lemma}
\label{lem:v_diff_between_best_and_empirical}
Let $\gpii$ and $\hpii$ be defined as in \cref{eq:v_def_fpii,eq:v_def_gpii}. Under the success event of \cref{lem:v_bernstein_general}, for any player $i \in [m]$ and $\pi \in \exti$, we have
\begin{align*}
\|\gpii - \hpii\|_{2,d_S} \leq 6 \Vmax \sqrt{C_A(\pi)\frac{ \log \frac{|\Gcal||\ext|}{\delta}}{n}} + 2\sqrt{\eps_{\Fcal}}.
\end{align*}
\end{lemma}
The proof is to invoke \cref{lem:v_bernstein_general} for $\gpii$ and $\hpii$ and the calculation is the same as \cref{lem:diff_between_best_and_empirical}.
Similar to \cref{lem:best_approximation_empirical_error}, we show that the best approximation of $V_i^\pi$ is contained in $\Gcal_i^{\pi,\beta_g}$.
\begin{lemma}
\label{lem:v_best_approximation_empirical_error}
Under the success event of \cref{lem:v_bernstein_general}, for any player $i \in [m]$ and $\pi \in \exti$, the following inequality for $\Ecal_i(\gpii,\pi;\Dcal)$ holds
\begin{align*}
\Ecal_i(\gpii,\pi;\Dcal) \leq \frac{80 C_A(\pi)\Vmax^2 \log \frac{|\Gcal||\ext|}{\delta}}{n} + 30 \eps_{\Fcal} \eqqcolon \beta_g.
\end{align*}
\end{lemma}
\begin{proof}
Applying \cref{lem:v_bernstein_general} and \cref{lem:v_diff_between_best_and_empirical}, we obtain
\begin{align}
&~ \bigg| \left\|\gpii - \tpii \gpii  \right\|_{2,d_S}^2 - \left\|\hpii - \tpii \gpii  \right\|_{2,d_S}^2 - \nonumber
\\
&~ \frac{1}{n} \sum_{(s,\bfa,\bfr,s') \in \Dcal} \frac{\pi(\bfa|s)}{d_A(\bfa|s)}\left(\gpii(s) - \bfr_i - \gamma \gpii(s') \right)^2 + \frac{1}{n} \sum_{(s,\bfa,\bfr,s') \in \Dcal} \frac{\pi(\bfa|s)}{d_A(\bfa|s)} \left(\hpii(s)  - \bfr_i - \gamma \gpii(s') \right)^2 \bigg| \nonumber
\\
\label{eq:v_fpi_g_expected_empirical}
\leq &~ 4 \Vmax \sqrt{C_A(\pi) \frac{ \log \frac{|\Gcal||\ext|}{\delta}}{n} \eps_{\Fcal}} + \frac{13 C_A(\pi) \Vmax^2 \log \frac{|\Gcal||\ext|}{\delta}}{n}. 
\end{align}
Similar to \cref{lem:best_approximation_empirical_error}, we bound $\|\gpii - \tpii \gpii  \|_{2,d_S}^2 - \|\hpii - \tpii \gpii  \|_{2,d_S}^2$ as follows,
\begin{align}
&~ \left\|\gpii - \tpii \gpii  \right\|_{2,d_S}^2 - \left\|\hpii - \tpii \gpii  \right\|_{2,d_S}^2 \nonumber
\\
\leq &~ \left( \left\|\gpii - \tpii \gpii  \right\|_{2,d_S} + \left\|\hpii - \tpii \gpii  \right\|_{2,d_S}\right) \left| \left\|\gpii - \tpii \gpii  \right\|_{2,d_S} - \left\|\hpii - \tpii \gpii  \right\|_{2,d_S} \right| \nonumber
\\
\leq &~ 36 \Vmax \sqrt{C_A(\pi) \frac{ \log \frac{|\Gcal||\ext|}{\delta}}{n} \eps_{\Fcal}} + 36 \Vmax^2  \frac{C_A(\pi) \log \frac{|\Gcal||\ext|}{\delta}}{n} + 8 \eps_{\Fcal}. \tag{By \cref{lem:v_diff_between_best_and_empirical}}\label{eq: v_diff_fpi_tcalfpi_g_tcalfpi}
\end{align}
Combining this with \cref{eq:v_fpi_g_expected_empirical}, we get
\begin{align*}
&~ \frac{1}{n} \sum_{(s,\bfa,\bfr,s') \in \Dcal} \frac{\pi(\bfa|s)}{d_A(\bfa|s)}\left(\gpii(s) - \bfr_i - \gamma \gpii(s') \right)^2 - \frac{1}{n} \sum_{(s,\bfa,\bfr,s') \in \Dcal} \frac{\pi(\bfa|s)}{d_A(\bfa|s)} \left(\hpii(s)  - \bfr_i - \gamma \gpii(s') \right)^2
\\
\leq &~ \left\|\gpii - \tpii \gpii  \right\|_{2,d_S}^2 - \left\|\hpii - \tpii \gpii  \right\|_{2,d_S}^2+4 \Vmax \sqrt{C_A(\pi) \frac{ \log \frac{|\Gcal||\ext|}{\delta}}{n} \eps_{\Fcal}} + \frac{13 C_A(\pi) \Vmax^2 \log \frac{|\Gcal||\ext|}{\delta}}{n} 
\\
\leq &~ \frac{80 C_A(\pi) \Vmax^2 \log \frac{|\Gcal||\ext|}{\delta}}{n} + 30 \eps_{\Fcal}. \tag{By AM-GM inequality}
\end{align*}
\end{proof}

We then prove that $g_i^{\pi,\max}(s_0)$ and $g_i^{\pi,\min}(s_0)$ are the upper bound and the lower bound on the value function $V_i^\pi(s_0)$ respectively.
\begin{restatable}{lemma}{minmax}\label{lem:v_minmax_error}
Under the success event of \cref{lem:v_bernstein_general}, for any player $i \in[m]$ and any $\pi \in \exti$, the following two inequalities hold
\begin{align*}
    g_i^{\pi,\max}(s_0) &\geq V_i^{\pi}(s_0) - \epsone \\
    g_i^{\pi,\min}(s_0) &\leq V_i^{\pi}(s_0) + \epsone.
\end{align*}
\end{restatable}
\begin{proof}
Let $\gpii$ be defined as in \cref{eq:v_def_fpii}, by invoking \cref{lem:v_evaluation_error}, we get
\begin{align*}
|\gpii(s_0)-V_i^\pi(s_0)| &\le \frac{\E_{s,\bfa \sim d^\pi,s' \sim P(\cdot|s,\bfa)}\left[g(s)-r_i(s,\bfa)-\gamma g(s')\right]}{1-\gamma} \\
&\le \frac{\|g-\tpii g\|_{2,d^\pi}}{1-\gamma} \le \epsone.
\end{align*}
By \cref{lem:v_best_approximation_empirical_error}, we know that $\gpii \in \Gcal^{\pi,\beta_g}_i$. Then, we obtain
\begin{align*}
    g_i^{\pi,\max}(s_0) \ge \gpii(s_0) \ge V_i^\pi(s_0) - \frac{\sqrt{\epsf}}{1-\gamma}.
\end{align*}
The case for $g_i^{\pi,\min}$ is similar.
\end{proof}
We now show that $\Ecal_i(g_i,\pi;\Dcal)$ could effectively estimate $\|g_i - \tpii g_i\|_{2,d_S}^2$.
\begin{lemma}
\label{lem:v_bellman_error}
Under the success event of \cref{lem:v_bernstein_general}, for any player $i \in [m]$ and any $\pi \in \exti$, given $\eps>0$, if $g_i \in \Gcal_i$ satisfies that $\Ecal_i(g_i,\pi;\Dcal) \leq \eps$, we have
\begin{align*}
\| g_i - \tpii g_i\|_{2,d_S} \leq 8\Vmax \sqrt{C_A(\pi)\frac{\log \frac{|\Gcal||\ext|}{\delta}}{n}} + 4\sqrt{\eps_{\Fcal,\Fcal}} + \sqrt{\eps}.
\end{align*}
\end{lemma}
\begin{proof}
Let $\hpii$ be defined as in \cref{eq:v_def_gpii}, let us define 
\begin{align}
\gpiid \coloneqq \argmin_{g'_i \in \Gcal_i} \left\| g'_i - \tpii g_i\right\|_{2,d_S}^2. \nonumber
\end{align}
Similar to \cref{lem:bellman_error}, we first upper bound $\|\hpii - \tpii g_i\|_{2,d_S}$. By invoking \cref{lem:v_bernstein_general}, we obtain, 
\begin{align*}
&~ \bigg| \left\| \hpii - \tpii g_i\right\|_{2,d_S}^2 - \left\| \gpiid - \tpii g_i\right\|_{2,d_S}^2 - \frac{1}{n} \sum_{(s,\bfa,\bfr,s') \in \Dcal} \frac{\pi(\bfa|s)}{d_A(\bfa|s)}\left(\hpii(s)  - \bfr_i - \gamma g_i(s') \right)^2
\\
&~ + \frac{1}{n} \sum_{(s,\bfa,\bfr,s') \in \Dcal} \frac{\pi(\bfa|s)}{d_A(\bfa|s)}\left(\gpiid(s)  - r - \gamma g_i(s') \right)^2\bigg|
\\
\leq &~ 2 \Vmax \|\hpii - \gpiid\|_{2,d_S}\sqrt{C_A(\pi) \frac{ \log \frac{|\Gcal||\ext|}{\delta}}{n}} + \frac{C_A(\pi) \Vmax^2 \log \frac{|\Gcal||\ext|}{\delta}}{n}.
\end{align*}
Rearranging the terms and by similar calculation to \cref{lem:bellman_error}, we have
\begin{align}
&~ \left\|\hpii - \tpii g_i\right\|_{2,d_S}^2 \nonumber
\\
\leq &~ 2 \Vmax \|\hpii - \tpii g_i\|_{2,d_S}\sqrt{C_A(\pi)\frac{ \log \frac{|\Gcal||\ext|}{\delta}}{n}} + \frac{2 C_A(\pi) \Vmax^2 \log \frac{|\Gcal||\ext|}{\delta}}{n} + 2\eps_{\Fcal,\Fcal}\label{eq:v_hpii_gpii}.
\end{align}
By solving \cref{eq:v_hpii_gpii}, we get
\begin{align}\label{eq:v_gpii_2_norm}
 \|\hpii - \tpii g_i\|_{2,d_S} \le  3 \Vmax \sqrt{C_A(\pi)\frac{ \log \frac{|\Gcal||\ext|}{\delta}}{n}} + \sqrt{2\eps_{\Fcal,\Fcal}}.
\end{align}
Then, we invoke \cref{lem:v_bernstein_general} for $\gpii$ and have
\begin{align*}
&~ \bigg| \| g_i - \tpii g_i\|_{2,d_S}^2 - \left\| \gpii - \tpii g_i\right\|_{2,d_S}^2
\\
&~ - \frac{1}{n} \sum_{(s,\bfa,\bfr,s') \in \Dcal} \frac{\pi(\bfa|s)}{d_A(\bfa|s)}\left(g_i(s) - \bfr_i - \gamma g_i(s') \right)^2 + \frac{1}{n} \sum_{(s,\bfa,\bfr,s') \in \Dcal} \frac{\pi(\bfa|s)}{d_A(\bfa|s)} \left(\gpii(s)  - \bfr_i - \gamma g_i(s') \right)^2 \bigg|
\\
\leq &~ 2 \Vmax \|g_i - \tpii g_i\|_{2,d_S} \sqrt{C_A(\pi)\frac{ \log \frac{|\Gcal||\ext|}{\delta}}{n}} + 3 \Vmax \sqrt{C_A(\pi) \frac{ \log \frac{|\Gcal||\ext|}{\delta}}{n} \varepsilon_{\Fcal,\Fcal}} + \frac{7 C_A(\pi) \Vmax^2 \log \frac{|\Gcal||\ext|}{\delta}}{n}.
\end{align*}
With similar calculation to \cref{lem:bellman_error}, we arrange the terms and have
\begin{align}
&~ \| g_i - \tpii g_i\|_{2,d_S}^2 \nonumber
\\
= &~ 2 \Vmax \|g_i - \tpii g_i\|_{2,d_S} \sqrt{C_A(\pi) \frac{ \log \frac{|\Gcal||\ext|}{\delta}}{n}} + 12 \Vmax \sqrt{C_A(\pi) \frac{ \log \frac{|\Gcal||\ext|}{\delta}}{n} \varepsilon_{\Fcal,\Fcal}} + \frac{16 C_A(\pi) \Vmax^2 \log \frac{|\Gcal||\ext|}{\delta}}{n} + 2\varepsilon_{\Fcal,\Fcal} + \varepsilon. \label{eq:v_bellman_error_f_mu}
\end{align}
Solving \cref{eq:bellman_error_f_mu} and using AM-GM inequality finishes the proof.
\end{proof}
Now we are ready to prove \cref{thm:v_multi_bound}.
\vmulti*
\begin{proof}
The proof for the first part is the same as \cref{thm:multi_bound}. For the second part, we invoke \cref{lem:v_evaluation_error} for $g_i^{\pi,\min}$ and $g_i^{\pi,\max}$
\begin{align*}
    &~ \gpimax(s_0)-\gpimin(s_0) \\
    = &~ \frac{1}{1 - \gamma} \underbrace{\Ebb_{d} \left[ \left(\gpimax - \tpii \gpimax \right) - \left(\gpimin - \tpii \gpimin \right)\right]}_{\text{(I)}} \\
    &~ + \frac{1}{1 - \gamma} \underbrace{\left(\Ebb_{d_\pi} \left[\Delta \gpii - \gamma P_i^{\pi} \Delta \gpii \right]  - \Ebb_{d} \left[\Delta \gpii - \gamma P_i^{\pi} \Delta \gpii \right] \right)}_{\text{(II)}}, \tag{$\Delta \gpii \coloneqq \gpimax - \gpimin$}
\end{align*}
where $d \in \Delta(\Scal)$ is an arbitrary distribution.
For the term (I), we have
\begin{align*}
\text{(I)} \leq &~ \left| \Ebb_{d} \left[ \left(\gpimax - \tpii \gpimax \right)\right] \right| + \left| \Ebb_{d} \left[ \left(\gpimin - \tpii \gpimin \right)\right] \right|
\\
\leq &~ \|\gpimax - \tpii \gpimax\|_{2,d} + \|\gpimin - \tpii \gpimin\|_{2,d} \tag{By Jensen's inequality}\\
\leq &~ \sqrt{\Cscr(d;d_S,\Gcal_i,\pi)}\left( \|\gpimax - \tpii \gpimax\|_{2,d_S} + \|\gpimin - \tpii \gpimin\|_{2,d_S}\right).
\end{align*}
Recall that $\beta_g=\frac{80 C_A(\pi) \Vmax^2\log\frac{|\Gcal||\ext|}{\delta}}{n}+30\epsf$. We invoke \cref{lem:bellman_error} and obtain with probability at least $1-\delta$
\begin{align}
\text{(I)} \leq \sqrt{\Cscr(d;d_S,\Gcal_i,\pi)}\order\left(\Vmax\sqrt{C_A(\pi) \frac{\log \frac{|\Gcal||\ext|}{\delta}}{n}}+\sqrt{\epsf+\epsff}\right). \label{eq:v_term_one}
\end{align}
For term (II), we have
\begin{align}
\text{(II)} \leq &~ \sum_{s \in \Scal} (d_{\pi}\setminus d)(s) \left[\Delta \gpii(s) - \gamma(P_i^{\pi}\Delta \gpii)(s)\right]
\nonumber \\
&~ + \sum_{(s) \in \Scal} \mathbb{I}(d(s) > d_{\pi}(s)) \left[d(s) - d_{\pi}(s)\right] \left|\Delta \gpii(s) - \gamma(P_i^{\pi}\Delta \gpii)(s)\right|
\nonumber \\
\leq &~ \sum_{s \in \Scal} (d_{\pi}\setminus d)(s) \left[\Delta \gpii(s) - \gamma(P_i^{\pi}\Delta \gpii)(s)\right] \nonumber \\
&~ + \Ebb_{d} \left[ \left|\gpimax - \tpii \gpimax \right| + \left|\gpimin - \tpii \gpimin \right|\right]
\nonumber \\
\leq &~\sum_{s \in  \Scal} (d_{\pi}\setminus d)(s) \left[\Delta \gpii(s) - \gamma(P_i^{\pi}\Delta \gpii)(s)\right] \nonumber \\
&+\sqrt{\Cscr(d;d_S,\Gcal_i,\pi)}\order\left(\Vmax\sqrt{C_A(\pi)\frac{\log \frac{|\Gcal||\ext|}{\delta}}{n}}+\sqrt{\epsf+\epsff}\right).\label{eq:v_term_two}
\end{align}
The last step is from the analysis of term (I). Combining \cref{eq:v_term_one} and \cref{eq:v_term_two}, we get
\begin{align*}
    \gpimax(s_0)-\gpimin(s_0) &\leq \min_d\frac{1}{1-\gamma}\sqrt{\Cscr(d;d_S,\Gcal_i,\pi)} \order\left(\Vmax\sqrt{C_A(\pi)\frac{\log \frac{|\Gcal||\ext|}{\delta}}{n}}+\sqrt{\epsf+\epsff}\right) \\
    &+\frac{1}{1-\gamma}\sum_{s \in \Scal}(d_{\pi}\setminus d)(s) \left[\Delta \gpii(s) - \gamma(P_i^{\pi}\Delta \gpii)(s)\right].
\end{align*}
This completes the proof.
\end{proof}

\section{Auxiliary Lemmas}

\begin{lemma}[Q-function Evaluation Error Lemma]\label{lem:evaluation_error}
For any player $i \in[m]$ and any $\pi \in \exti$, and any $f \in \mathbb{R}^{\Scal \times \Acal}$
\begin{align*}
    f(s_0,\pi)-V_i^{\pi}(s_0)=\frac{\E_{s,\bfa \sim d^\pi,s' \sim P(\cdot|s,\bfa)}\left[f(s,\bfa)-r_i(s,\bfa)-\gamma f(s',\pi)\right]}{1-\gamma}
\end{align*}
\end{lemma}
\begin{proof}
We observe that
\begin{align*}
    &~ \sum_{s,\bfa}\sum_{t=0}^{\infty}\gamma^{t+1}\Pr(s_t=s,\bfa_t=\bfa|s_0,\pi)\sum_{s'}\Pr(s_{t+1}=s'|s_t=s,\bfa_t=\bfa)f(s',\pi) \\
    = &~ \sum_{s,\bfa}\sum_{t=1}^{\infty}\gamma^t \Pr(s_t=s,\bfa_t=\bfa|s_0,\pi)f(s,\bfa)
\end{align*}
Then, we have
\begin{align*}
&~ \frac{\E_{s,\bfa \sim d^\pi,s' \sim P(\cdot|s,\bfa)}\left[f(s,\bfa)-\gamma f(s',\pi)\right]}{1-\gamma} \\
= &~ \sum_{s,\bfa}\sum_{t=0}^{\infty} \gamma^t \Pr(s_t=s,\bfa_t=\bfa|s_0,\pi)f(s,\bfa)-\sum_{s,\bfa}\sum_{t=1}^{\infty}\gamma^t \Pr(s_t=s,\bfa_t=\bfa|s_0,\pi)f(s,\bfa) \\
= &~ \sum_{\bfa}\Pr(\bfa_0=\bfa|s_0,\pi)f(s_0,\bfa)=f(s_0,\pi).
\end{align*}
Since $V_i^{\pi}(s_0)=\frac{\E_{d^\pi}[r_i(s,\bfa)]}{1-\gamma}$, rearranging the terms finishes the proof.
\end{proof}

\begin{lemma}[Value Function Evaluation Error Lemma]\label{lem:v_evaluation_error}
For any player $i \in[m]$ and any $\pi \in \exti$, and any $f \in \mathbb{R}^{\Scal}$
\begin{align*}
    f(s_0)-V_i^{\pi}(s_0)=\frac{\E_{s,\bfa \sim d^\pi,s' \sim P(\cdot|s,\bfa)}\left[f(s)-r_i(s,\bfa)-\gamma f(s')\right]}{1-\gamma}
\end{align*}
\end{lemma}
\begin{proof}
We observe that
\begin{align*}
    &~ \sum_{s,\bfa}\sum_{t=0}^{\infty}\gamma^{t+1}\Pr(s_t=s,\bfa_t=\bfa|s_0,\pi)\sum_{s'}\Pr(s_{t+1}=s'|s_t=s,\bfa_t=\bfa)f(s') \\
    = &~ \sum_{s,\bfa}\sum_{t=1}^{\infty}\gamma^t \Pr(s_t=s,\bfa_t=\bfa|s_0,\pi)f(s)
\end{align*}
Then, we have
\begin{align*}
&~ \frac{\E_{s \sim d^\pi,s' \sim P(\cdot|s,\bfa)}\left[f(s)-\gamma f(s')\right]}{1-\gamma} \\
= &~ \sum_{s,\bfa}\sum_{t=0}^{\infty} \gamma^t \Pr(s_t=s,\bfa_t=\bfa|s_0,\pi)f(s)-\sum_{s,\bfa}\sum_{t=1}^{\infty}\gamma^t \Pr(s_t=s,\bfa_t=\bfa|s_0,\pi)f(s) \\
= &~ \sum_{\bfa}\Pr(\bfa_0=\bfa|s_0,\pi)f(s_0)=f(s_0).
\end{align*}
Since $V_i^{\pi}(s_0)=\frac{\E_{d^\pi}[r_i(s,\bfa)]}{1-\gamma}$, rearranging the terms finishes the proof.
\end{proof}

\end{document}